\documentclass{amsart}%
\usepackage{amssymb}
\usepackage{amsfonts}
\usepackage{amsmath}
\usepackage{graphicx}%
\newtheorem{theorem}{Theorem}
\theoremstyle{plain}

\newtheorem{corollary}{Corollary}

\newtheorem{definition}{Definition}

\newtheorem{lemma}{Lemma}
\newtheorem{notation}{Notation}

\newtheorem{proposition}{Proposition}
\newtheorem{remark}{Remark}

\numberwithin{equation}{section}
\begin{document}
\title[Deep Neural Networks]{Deep Neural Networks: A Formulation Via Non-Archimedean Analysis}
\author{W. A. Z\'{u}\~{n}iga-Galindo}
\address{University of Texas Rio Grande Valley\\
School of Mathematical \& Statistical Sciences\\
One West University Blvd\\
Brownsville, TX 78520, United States}
\email{wilson.zunigagalindo@utrgv.edu}
\thanks{The author was partially supported by the Lokenath Debnath Endowed Professorship.}
\subjclass{Primary 68T07, 65D15; Secondary 41A30, 11S85}
\keywords{Artificial neural networks and deep learning, Algorithms for approximation of
functions, non-Archimedean fields, p-adic analysis.}

\begin{abstract}
We introduce a new class of deep neural networks (DNNs) with multilayered
tree-like architectures. The architectures are codified by using numbers from
the ring of integers of non-Archimedean local fields. These rings have a
natural hierarchical organization as infinite rooted trees. Natural morphisms
on these rings allow us to construct finite multilayered architectures. The
new DNNs are robust universal approximators of real-valued functions defined
on the mentioned rings. We also show that the DNNs are robust universal
approximators of real-valued square-integrable functions defined in the unit interval.

\end{abstract}
\maketitle

\section{Introduction}

It is known that neural networks (NN) can learn hierarchically. However, there
is currently no theoretical framework that explains how this hierarchical
learning occurs; see, e.g., \cite{Abbet et al}, and the references therein.
The study of deep learning architectures that can learn from hierarchically
organized data is an active research area nowadays; see, e.g., \cite{Abbet et
al}, \cite{Zhang et al}, and the references therein. In \cite{Abbet et al},
the authors propose to find natural classes of hierarchical functions and to
study how regular deep neural networks (DNNs) can learn them. In this paper,
we follow the converse approach. Using non-Archimedean analysis, we construct
a new class of DNNs with tree-like architectures. These DNNs are robust
universal approximators of hierarchical functions and also universal
approximators of standard functions. The new hierarchical DNNs compute
approximations using arithmetic operations in non-Archimedean fields, but they
can be trained using the traditional backpropagation method.

A\ non-Archimedean vector space $\left(  M,\left\Vert \cdot\right\Vert
\right)  $ is a normed vector space whose norm satisfies
\[
\left\Vert x+y\right\Vert \leq\max\left\{  \left\Vert x\right\Vert ,\left\Vert
y\right\Vert \right\}  ,
\]
for any two vectors $x$, $y$ in $M$.\ In such a space, the balls are organized
hierarchically. This type of space plays a central role in formulating models
of complex multi-level systems; in such systems, the hierarchy is significant;
see, e.g., \cite{Iordache}-\cite{KKZuniga}, and the references therein. These
systems are composed of several subsystems and exhibit emergent behavior
resulting from nonlinear interactions across multiple levels of organization.
The field of $p$-adic numbers $\mathbb{Q}_{p}$ and the field of formal Laurent
series $\mathbb{F}_{p}\left(  \left(  T\right)  \right)  $ are paramount
examples of non-Archimedean vector spaces.

Nowadays, it is widely accepted that cortical neural networks are arranged in
fractal or self-similar patterns and have the small-world property, see, e.g.,
\cite{Hilgetag et al}, \cite{Sporns}, and the references therein. Methods of
non-Archimedean analysis have been successfully used to construct models for
this type of network, see, e.g., \cite{Khrennikov}, \cite{Zuniga-Zambrano},
and the references therein. A restricted Boltzmann machine (RBM) is a
generative stochastic artificial neural network based on the dynamics of a
spin glass. The $p$-adic spin glasses constitute a particular case of
hierarchical spin glasses. The RBM corresponding to a $p$-adic spin glass\ is
a hierarchical neural network. These neural networks are universal
approximators, \cite{Zuniga-DBNs}, and their dynamics can be understood via
Euclidean quantum field theory, \cite{Zuniga et al}, \cite{Zuniga-ATMP}.
$p$-Adic versions of the cellular networks were developed
in\ \cite{Zambrano-Zuniga-1}-\cite{Zambrano-Zuniga-2}. This work continues our
investigation of hierarchical NNs via non-Archimedean analysis.

To discuss our results, here, we restrict to the case of the ring of integers
$\mathbb{F}_{p}\left[  \left[  T\right]  \right]  $ of the field of formal
Laurent series with coefficients in a finite field $\mathbb{F}_{p}=\left\{
0,\ldots,p-1\right\}  $, with $p$ a prime number,
\[
\mathbb{F}_{p}\left[  \left[  T\right]  \right]  =\left\{  T^{r}%
{\displaystyle\sum\limits_{k=0}^{\infty}}
a_{k}T^{k};a_{k}\in\mathbb{F}_{p}\text{, }a_{0}\neq0,r\in\mathbb{Z}\right\}
\cup\left\{  0\right\}  ,
\]
where $T$ is an indeterminate. The function $\left\vert T^{r}\sum
_{k=0}^{\infty}a_{k}T^{k}\right\vert =p^{-r}$, $\left\vert 0\right\vert =0$
defines a norm. Then $\left(  \mathbb{F}_{p}\left[  \left[  T\right]  \right]
,\left\vert \cdot\right\vert \right)  $ is a non-Archimedean vector space,
which is also a complete space, and $\left(  \mathbb{F}_{p}\left[  \left[
T\right]  \right]  ,+\right)  $ is a compact topological group. We denote by
$dx$ a Haar measure in $\mathbb{F}_{p}\left[  \left[  T\right]  \right]  $.
The points of $\mathbb{F}_{p}\left[  \left[  T\right]  \right]  $ are
naturally organized in an infinite rooted tree. The points of the form
\[
G_{l}=\left\{  a_{0}+\ldots+a_{l-1}T^{l-1};a_{k}\in\mathbb{F}_{p}\right\}
\subset\mathbb{F}_{p}\left[  \left[  T\right]  \right]  ,\text{ }l\geq1,
\]
constitute the $l$-th layer (level) of the tree $\mathbb{F}_{p}\left[  \left[
T\right]  \right]  $. Geometrically, $G_{l}$ is a finite rooted tree, but
also, $G_{l}$ is a finite additive group. We assume that \ at every point of
$G_{l}$ there is a neuron. In this way, we will identify $\mathbb{F}%
_{p}\left[  \left[  T\right]  \right]  $ with an infinite hierarchical
network; $G_{l}$ is a discrete approximation of $\mathbb{F}_{p}\left[  \left[
T\right]  \right]  $ obtained by cutting the infinite \ tree $\mathbb{F}%
_{p}\left[  \left[  T\right]  \right]  $\ at layer $l$.

A ball in $\mathbb{F}_{p}\left[  \left[  T\right]  \right]  $ with center at
$a\in G_{L}$, $L\geq1$, and radius $p^{-L}$ is the set $B_{-L}(a)=a+T^{L}%
\mathbb{F}_{p}\left[  \left[  T\right]  \right]  $. This ball is an infinite
rooted tree, with root $a$. We denote by $\Omega\left(  p^{L}\left\vert
x-a\right\vert \right)  $ the characteristic function of the ball $B_{-L}(a)$. 

The set of functions%
\[
\varphi\left(  x\right)  =%
{\displaystyle\sum\limits_{a\in G_{l}}}
c_{a}\Omega\left(  p^{l}\left\vert x-a\right\vert \right)  \text{, }c_{a}%
\in\mathbb{R},
\]
form an $\mathbb{R}$-vector space, denoted as $\mathcal{D}^{l}=\mathcal{D}%
^{l}\left(  \mathbb{F}_{p}\left[  \left[  T\right]  \right]  \right)  $. The
functions from this space are naturally interpreted as hierarchical functions.
The $\mathbb{R}$-vector space $\mathcal{D}=\cup_{l}\mathcal{D}^{l}$ is called
the space of test functions; this space is dense in $L^{\rho}\left(
\mathbb{F}_{p}\left[  \left[  T\right]  \right]  ,dx\right)  $, $\rho
\in\left[  1,\infty\right]  $.

We pick an activation function $\sigma_{M}:\mathbb{R}\rightarrow\left(
-M,M\right)  $, with $M>0$, a weight function $w\in L^{\infty}\left(
\mathbb{F}_{p}\left[  \left[  T\right]  \right]  \times\mathbb{F}_{p}\left[
\left[  T\right]  \right]  ,dx\text{ }dy\right)  $,\ and bias function
$\theta\in L^{\infty}\left(  \mathbb{F}_{p}\left[  \left[  T\right]  \right]
,dx\right)  $. A continuous DNN on $\mathbb{F}_{p}\left[  \left[  T\right]
\right]  $ is a function of type%
\begin{equation}
Y(x)=\sigma_{M}\left(  \text{ }\int\limits_{\mathbb{F}_{p}\left[  \left[
T\right]  \right]  }w(x,y)X(y)d^{N}y+\theta\left(  x\right)  \right)
,\label{Perceptron}%
\end{equation}
where $X\in L^{\infty}\left(  \mathbb{F}_{p}\left[  \left[  T\right]  \right]
,dx\right)  $, is the state of the DNN. At first sight, (\ref{Perceptron})
seems to be a continuous perceptron, but \ due to the fact that $\mathbb{F}%
_{p}\left[  \left[  T\right]  \right]  $\ has a hierarchical structure,
(\ref{Perceptron}) is a hierarchical DNN. For practical purposes, only
discrete versions of (\ref{Perceptron}) are relevant. Here, we consider DNNs
where $X\in\mathcal{D}^{l-1}\left(  \mathbb{F}_{p}\left[  \left[  T\right]
\right]  \right)  $, $\theta\in\mathcal{D}^{l}\left(  \mathbb{F}_{p}\left[
\left[  T\right]  \right]  \right)  $, $w\in\mathcal{D}^{l}\left(
\mathbb{F}_{p}\left[  \left[  T\right]  \right]  \times\mathbb{F}_{p}\left[
\left[  T\right]  \right]  \right)  $, for $l\geq2$, then $Y\in\mathcal{D}%
^{l}\left(  \mathbb{F}_{p}\left[  \left[  T\right]  \right]  \right)  $, and
\begin{equation}
Y\left(  b\right)  =\sigma_{M}\left(
{\displaystyle\sum\limits_{a\in G_{l}}}
X\left(  \Lambda_{l}\left(  a\right)  \right)  w\left(  b,a\right)
+\theta\left(  b\right)  \right)  \text{, for }b\in G_{l}\text{,}%
\label{Perceptron_Discrete}%
\end{equation}
where $\Lambda_{l}:G_{l}\rightarrow G_{l-1}$ is\ the natural group
homomorphism. Using (\ref{Perceptron_Discrete}), we construct hierarchical
DNNs with neurons organized in a tree-like structure, with multiple layers.
These DNNs can be trained using the standard backpropagation method, see Lemma
\ref{Lemma_Key}. These DNNs are robust universal approximators: given $f\in
L^{\rho}\left(  \mathbb{F}_{p}\left[  \left[  T\right]  \right]  ,dx\right)
$, with $\left\Vert f\right\Vert _{\rho}<M$, $\rho\in\left[  1,\infty\right]
$, for any input $X\in\mathcal{D}^{L}\left(  \mathbb{F}_{p}\left[  \left[
T\right]  \right]  \right)  $, $L\geq1$, there exist $\sigma_{M}%
,w\in\mathcal{D}^{L+\Delta}\left(  \mathbb{F}_{p}\left[  \left[  T\right]
\right]  \times\mathbb{F}_{p}\left[  \left[  T\right]  \right]  \right)
,\theta\in\mathcal{D}^{L+\Delta}\left(  \mathbb{F}_{p}\left[  \left[
T\right]  \right]  \right)  $, where $\Delta$ is the depth of the network,
such that the output of the network $Y\in\mathcal{D}^{L+\Delta}\left(
\mathbb{F}_{p}\left[  \left[  T\right]  \right]  \right)  $ satisfies
$\left\Vert Y-f\right\Vert _{\rho}<\epsilon$; furthermore, this approximation
remains valid for $w$ and $\theta$ belonging to small balls, see Theorem
\ref{Theorem1}. These results can be extended to approximations of a finite
set of functions, and \ functions defined on arbitrary compact subsets of
$\mathbb{F}_{p}\left(  \left(  T\right)  \right)  $; see Corollary
\ref{Corollary2}\ and Theorem \ref{Theorem1A}. It is relevant to mention that
the results mentioned are still valid if $w\left(  b,a\right)  =w\left(
b-a\right)  $ in (\ref{Perceptron_Discrete}); this case corresponds to a
convolutional DNN.

A natural question is if the hierarchical DNNs can approximate
non-hierarchi\-cal functions, i.e., functions from $L^{\rho}\left(  \left[
0,1\right]  ,dt\right)  $, with $\left\Vert f\right\Vert _{\rho}<M$, $\rho
\in\left[  1,\infty\right]  $, where $dt$ denotes the Lebesgue measure on
$\mathbb{R}$. It turns out that $L^{\rho}\left(  \left[  0,1\right]
,dt\right)  \simeq L^{\rho}\left(  \mathbb{F}_{p}\left[  \left[  T\right]
\right]  ,dx\right)  $, where $\simeq$\ denotes a linear surjective isometry;
see Theorem \ref{Theorem2}. Therefore, hierarchical DNNs are robust universal
approximators of non-hierarchical functions, see Theorem \ref{Theorem3A}.

In \cite{Abbet et al}, the authors use the Fourier-Walsh series to approximate
hierarchi\-cal functions. The isometry $L^{2}\left(  \left[  0,1\right]
,dt\right)  \simeq L^{2}\left(  \mathbb{F}_{p}\left[  \left[  T\right]
\right]  ,dx\right)  $ maps orthonormal bases of $L^{2}\left(  \mathbb{F}%
_{p}\left[  \left[  T\right]  \right]  ,dx\right)  $ into orthonormal bases of
$L^{2}\left(  \left[  0,1\right]  ,dt\right)  $, cf. Theorem \ref{Theorem4}.
By choosing $p=2$, \ and a suitable orthonormal basis in $L^{2}\left(
\mathbb{F}_{p}\left[  \left[  T\right]  \right]  ,dx\right)  $, we recover the
classical Fourier-Walsh series. Then, the hierarchical functions used in
\cite{Abbet et al} are elements of $L^{2}\left(  \mathbb{F}_{p}\left[  \left[
T\right]  \right]  ,dx\right)  $, and also, the notion hierarchy used there
has a non-Archimedean origin.

Our approach is inspired by the classical result asserting that the
perceptrons are universal approximators; see, e.g., \cite[Chapter 9]{Calin},
\cite{Cybenko}-\cite{Grivonal et al}, \cite{HORNIK}, \cite{Pinkus}, among many
references. However, (\ref{Perceptron_Discrete}) is not a classical
perceptron; see Remark \ref{Remark_Cybenko}. Furthermore, the input to our
DNNs is a test function, i.e., it has a hierarchical representation. We have
implemented some $p$-adic NNs for image processing; see
\cite{Zambrano-Zuniga-2}, \cite{Zuniga et al}. In this application, the input
image is transformed into a test function (a finite weighted tree), and at the
end of the calculation, the output, a test function, is transformed into an image.

Several computation models involving $p$-adic numbers have been studied. In
\cite{Maller et al}, a model of computation over the $p$-adic numbers, for odd
primes $p$, is defined following the approach of Blum, Shub, and Smale
\cite{Blum et al}. NNs whose states are $p$-adic numbers were studied in
\cite{Albeverio et al}, and $p$-adic automata in \cite{Anashin}. In
\cite{Khrennikov}, \cite{Shor et al}-\cite{Short et al 2}, $p$-adic models for
human brain activity were developed.

Finally, the potential implementations and applications of the DNNs introduced
here are beyond the scope of this paper, which has a purely mathematical
nature. However, the relevance of the concepts and techniques introduced here
is crucial for demonstrating the rigorous existence of the thermodynamic limit
for deep neural networks and recurrent neural networks, assuming that the
activation functions are sigmoids \cite{Zuniga-Utimo}. The mentioned networks
are modeled as lattice statistical field theories using billions of
parameters, see \cite{Roberts et al}-\cite{Segadlo et al}, but the study of
qualitative behavior, like critical organization, requires a thermodynamic
limit which is a continuous neural network, where the neurons form a
continuous space with infinitely many points.

\section{Preliminary results}

We quickly review some essential aspects of non-Archimedean analysis required
in the article and fix some notations to be used here.

\subsection{Non-Archimedean local fields}

A non-Archimedean local field $\mathbb{K}$ of arbitrary characteristic is a
locally compact topological field with respect to a non-discrete topology,
which is induced by a norm (or absolute value) $\left\vert \cdot\right\vert
_{\mathbb{K}}$ satisfying%
\[
\left\vert x+y\right\vert _{\mathbb{K}}\leq\max\left\{  \left\vert
x\right\vert _{\mathbb{K}},\left\vert y\right\vert _{\mathbb{K}}\right\}
\text{ for }x,y\in\mathbb{K},
\]
i.e., $\left\vert \cdot\right\vert _{\mathbb{K}}$ is a non-Archimedean norm.
For an in-depth exposition, the reader may consult \cite{We}, see also
\cite{A-K-S}, \cite{Kochubei}, \cite{Taibleson}-\cite{V-V-Z}.

The ring of integers $\mathcal{O}_{\mathbb{K}}$ of $\mathbb{K}$ is
\[
\mathcal{O}_{\mathbb{K}}=\left\{  x\in\mathbb{K}\text{ ; }|x|_{\mathbb{K}}%
\leq1\right\}  .
\]
This is a valuation ring with a unique maximal ideal $P_{\mathbb{K}}$. In
terms of the absolute value $|\cdot|_{\mathbb{K}}$, $P_{\mathbb{K}}$ can be
described as
\[
\mathcal{P}_{\mathbb{K}}=\left\{  x\in\mathbb{K}\text{ ; }|x|_{\mathbb{K}%
}<1\right\}  .
\]
Let $\overline{\mathbb{K}}=\mathcal{O}_{\mathbb{K}}/\mathcal{P}_{\mathbb{K}}$
the residue field of $\mathbb{K}$. Thus $\overline{\mathbb{K}}$ $=$
$\mathbb{F}_{q}$, the finite field with $q$ elements. Let $\mathfrak{\wp}$ be
a fixed generator of $\mathcal{P}_{\mathbb{K}}$, $\mathfrak{\wp}$ is called a
uniformizing parameter of $\mathbb{K},$ then $\mathcal{P}_{\mathbb{K}%
}=\mathfrak{\wp}\mathcal{O}_{\mathbb{K}}$. Furthermore, we assume that
$|\mathfrak{\wp}|_{\mathbb{K}}=q^{-1}$. For $z\in\mathbb{K}$, $ord(z)\in$
$\mathbb{Z}\cup\left\{  +\infty\right\}  $ denotes the valuation of $z$, and
$|z|_{\mathbb{K}}=q^{-ord(z)}$. If $z\in\mathbb{K}\backslash\left\{
0\right\}  $, then $ac(z)=z\pi^{-ord(z)}$ denotes the angular component of $z$.

The natural map $\mathcal{O}_{\mathbb{K}}\rightarrow\mathcal{O}_{\mathbb{K}%
}/\mathcal{P}_{\mathbb{K}}\simeq\mathbb{F}_{q}$ is called the reduction
$\operatorname{mod}$ $P_{\mathbb{K}}$. We fix $\mathfrak{S}\subset
\mathcal{O}_{\mathbb{K}}$ a set of representatives of $\mathbb{F}_{q}$ in
$\mathcal{O}_{\mathbb{K}}$, i.e., $\mathfrak{S}$ is mapped bijectively into
$\mathbb{F}_{q}$ by the reduction $\operatorname{mod}$ $P_{\mathbb{K}}$. We
assume that $0\in\mathfrak{S}$. Any non-zero element $x$ of $\mathbb{K}$ can
be written as
\[
x=\mathfrak{\wp}^{ord(x)}\sum\limits_{i=0}^{\infty}x_{i}\mathfrak{\wp}%
^{i}\text{, }x_{_{i}}\in\mathfrak{S}\text{, and }x_{0}\neq0\text{.}%
\]
This series converges in the norm $\left\vert \cdot\right\vert _{\mathbb{K}}$.

Along this article we assume that $q=p$ a prime number. In this case,
\[
\mathfrak{S=}\left\{  0,1,\ldots,p-1\right\}  .
\]
The field $\mathbb{K}$ is isomorphic to\ the field of $p$-adic numbers
$\mathbb{Q}_{p}$, or to the field of formal Laurent series $\mathbb{F}%
_{p}\left(  \left(  T\right)  \right)  $ with coefficients in a finite field
$\mathbb{F}_{p}$ with $p$ elements.

The field of $p$-adic numbers $\mathbb{Q}_{p}$ is obtained as the completion
of the field of rational numbers $\mathbb{Q}$ with respect to the $p$-adic
norm $|\cdot|_{p}$, which is defined as
\[
\left\vert x\right\vert _{p}=\left\{
\begin{array}
[c]{lll}%
0 & \text{if} & x=0\\
&  & \\
p^{-\gamma} & \text{if} & x=p^{\gamma}\frac{a}{b}\text{,}%
\end{array}
\right.
\]
where $a$ and $b$ are integers coprime with $p$. The integer $\gamma:=ord(x)$,
with $ord(0):=+\infty$, is called the\ $p$-adic order of $x$. $\mathbb{Q}_{p}$
is a field of characteristic zero, while its residue field $\mathbb{F}_{p}$
has characteristic $p>0$. \ The prime $p$ is a local parameter, and the set
$\mathbb{\ }\mathfrak{S=}\left\{  0,1,\ldots,p-1\right\}  \mathfrak{\subset
}\mathbb{Q}_{p}$ is not a subfield of $\mathbb{Q}_{p}$.

A series of the form $%
{\displaystyle\sum\limits_{k=k_{0}}^{\infty}}
a_{k}T^{k}$, with $k_{0}\in\mathbb{Z}$, $a_{k}\in\mathbb{F}_{p}$, is called a
formal\ Laurent series. These series form a field $\mathbb{F}_{p}\left(
\left(  T\right)  \right)  $. The standard norm on $\mathbb{F}_{p}\left(
\left(  T\right)  \right)  $ is defined as%
\[
\left\vert x\right\vert =\left\{
\begin{array}
[c]{lll}%
0 & \text{if} & x=0\\
p^{-k_{0}} & \text{if} & x=%
{\displaystyle\sum\limits_{k=k_{0}}^{\infty}}
a_{k}T^{k}\text{, }a_{k_{0}}\neq0.
\end{array}
\right.
\]
Notice that $\left\vert T\right\vert =p^{-1}$.

\subsection{The space of test functions}

For $r\in\mathbb{Z}$, denote by
\[
B_{r}(a)=\{x\in\mathbb{K};|x-a|_{\mathbb{K}}\leq p^{r}\},
\]
the ball of radius $p^{r}$ with center at $a\in\mathbb{K}$, and take
$B_{r}(0):=B_{r}$. The ball $B_{0}$ equals to $\mathcal{O}_{\mathbb{K}}$, the
ring of integers of $\mathbb{K}$. The balls are both open and closed subsets
in $\mathbb{K}$. In addition, two balls in $\mathbb{K}$ are either disjoint or
one is contained in the other. A subset of $\mathbb{K}$ is compact if and only
if it is closed and bounded in $\mathbb{K}$, see, e.g., \cite[Section
1.8]{A-K-S} or \cite[Section 1.3]{V-V-Z}. The balls and spheres are compact
subsets. Thus $\left(  \mathbb{K},|\cdot|_{\mathbb{K}}\right)  $ is a locally
compact topological space.

Since $(\mathbb{K},+)$ is a locally compact topological group, there exists a
Haar measure $dx$, which is invariant under translations, i.e., $d(x+a)=dx$,
\cite{Halmos}. If we normalize this measure by the condition $\int
_{\mathcal{O}_{\mathbb{K}}}dx=1$, then $dx$ is unique.

\begin{notation}
(i) We will use $\Omega\left(  p^{-r}|x-a|_{\mathbb{K}}\right)  $ to denote
the characteristic function of the ball $B_{-r}(a)=a+\mathfrak{\wp}%
^{r}\mathcal{O}_{\mathbb{K}}$, where $\left\vert \mathfrak{\wp}\right\vert
_{\mathbb{K}}=p^{-1}$, and $\mathcal{O}_{\mathbb{K}}$ is the unit ball.
Notice that
\[
\mathfrak{\wp}=\left\{
\begin{array}
[c]{lll}%
p & \text{if} & \mathbb{K}=\mathbb{Q}_{p}\\
&  & \\
T & \text{if} & \mathbb{K}=\mathbb{F}_{p}\left(  \left(  T\right)  \right)  .
\end{array}
\right.
\]
For more general sets, we will use the notation $1_{A}$ for the characteristic
function of set $A$.
\end{notation}

A real-valued function $\varphi$ defined on $\mathbb{K}$ is called locally
constant if for any $x\in\mathbb{K}$ there exist an integer $l(x)\in
\mathbb{Z}$ such that%
\begin{equation}
\varphi(x+x^{\prime})=\varphi(x)\text{ for any }x^{\prime}\in B_{l(x)}.
\label{local_constancy}%
\end{equation}
A function $\varphi:\mathbb{K}\rightarrow\mathbb{R}$ is called a
Bruhat-Schwartz function\textit{\ }(or a test function) if it is locally
constant with compact support. Any test function can be represented as a
linear combination, with real coefficients, of characteristic functions of
balls. The $\mathbb{R}$-vector space of Bruhat-Schwartz functions is denoted
by $\mathcal{D}(\mathbb{K})$. For $\varphi\in\mathcal{D}(\mathbb{K})$, the
largest number $l=l(\varphi)$ satisfying (\ref{local_constancy}) is called the
exponent of local constancy (or the parameter of constancy) of $\varphi$. We
denote by $\mathcal{D}(\mathcal{O}_{\mathbb{K}})$ the subspace of the test
functions supported in the unit ball.

\section{Non-Archimedean deep neural networks}

We set $L^{\rho}(\mathcal{O}_{\mathbb{K}})$, $\rho\in\left[  1,\infty\right]
$, for the $\mathbb{R}$-vector space of functions $f:\mathcal{O}_{\mathbb{K}%
}\rightarrow\mathbb{R}$ satisfying%
\[
\left\Vert f\right\Vert _{\rho}=\left\{
\begin{array}
[c]{ll}%
\left(  \text{ }%
{\displaystyle\int\limits_{\mathcal{O}_{\mathbb{K}}}}
\left\vert f\left(  x\right)  \right\vert ^{\rho}dx\right)  ^{\frac{1}{\rho}%
}<\infty, & \text{for }\rho\in\left[  1,\infty\right)  \\
& \\
\left\Vert f\right\Vert _{\infty}<\infty & \text{for }\rho=\infty,
\end{array}
\right.
\]
where $\left\Vert f\right\Vert _{\infty}$ denotes the essential supremum of
$\left\vert f\left(  x\right)  \right\vert $. The H\"{o}lder inequality and
$\int_{\mathcal{O}_{\mathbb{K}}}d^{N}x=1$ imply that%
\[
\left\Vert f\right\Vert _{1}\leq\left\Vert f\right\Vert _{\rho}\text{, for
}1\leq\rho\leq\infty,
\]
which in turn implies that
\begin{equation}
L^{\rho}(\mathcal{O}_{\mathbb{K}})\hookrightarrow L^{1}(\mathcal{O}%
_{\mathbb{K}})\text{, }1\leq\rho\leq\infty,\label{Embedding_0}%
\end{equation}
where the arrow denotes a continuous embedding. Since $\mathcal{D}%
(\mathcal{O}_{\mathbb{K}})\subset L^{\rho}(\mathcal{O}_{\mathbb{K}})$, for
$1\leq\rho\leq\infty$, and that $\mathcal{D}(\mathcal{O}_{\mathbb{K}})$ is
dense in $L^{\rho}(\mathcal{O}_{\mathbb{K}})$, for $1\leq\rho<\infty$, see
\cite[Chap. I, Proposition 1.3]{Taibleson}; $\mathcal{D}(\mathcal{O}%
_{\mathbb{K}})$ also dense in $L^{\infty}(\mathcal{O}_{\mathbb{K}})$.

We fix a sigmoidal function (or an activation function) $\sigma:\mathbb{R}%
\rightarrow\mathbb{R}$, which is a surjective, bounded, and differentiable
function. We assume that
\begin{equation}
\sigma:\mathbb{R}\rightarrow\left(  -1,1\right)  . \label{Eq_Constant_M}%
\end{equation}
The calculations with neural networks required using scaled versions of this
function: $\sigma_{M}:=M\sigma:\mathbb{R}\rightarrow\left(  -M,M\right)  $,
$M>1$.

We set \ $w\left(  x,y\right)  \in L^{\infty}(\mathcal{O}_{\mathbb{K}}%
\times\mathcal{O}_{\mathbb{K}})$ and assume that%
\[
w\left(  \cdot,y\right)  \in\mathcal{D}(\mathcal{O}_{\mathbb{K}}),
\]
where the index of local constancy of the function $w\left(  \cdot,y\right)  $
is independent of $y$. We also set $\theta\left(  x\right)  \in L^{\infty
}(\mathcal{O}_{\mathbb{K}})$.

\begin{definition}
Assuming that $\sigma_{M}$, $w,$ $\theta$ are as above, and $X\in
L^{1}(\mathcal{O}_{\mathbb{K}})$, the function%
\begin{equation}
Y(x)=Y(x;\mathcal{O}_{\mathbb{K}},\sigma_{M},w,\theta):=\sigma_{M}\left(
\text{ }\int\limits_{\mathcal{O}_{\mathbb{K}}}w(x,y)X(y)dy+\theta\left(
x\right)  \right)  \in L^{\infty}(\mathcal{O}_{\mathbb{K}})\label{DBN}%
\end{equation}
is called a continuos non-Archimedean deep neural network (DNN) with
activation function $\sigma_{M}$, kernel $w$, and bias $\theta$. The
real-valued function $X(x)\in\mathbb{R}$ describes the state of the neuron
$x\in\mathcal{O}_{\mathbb{K}}$. We denote such neural network as
$DNN(\mathcal{O}_{\mathbb{K}},\sigma_{M},w,\theta,\infty)$.
\end{definition}

The network has infinitely many neurons organized in a rooted tree. Notice
that%
\begin{multline*}
\left\vert \text{ }\int\limits_{\mathcal{O}_{\mathbb{K}}}w(x,y)X(y)dy+\theta
\left(  x\right)  \right\vert \leq\left\vert \text{ }\int\limits_{\mathcal{O}%
_{\mathbb{K}}}w(x,y)X(y)dy\right\vert +\left\vert \theta\left(  x\right)
\right\vert \\
\leq\left\Vert w\right\Vert _{\infty}\left\Vert X\right\Vert _{1}+\left\Vert
\theta\right\Vert _{\infty}\text{.}%
\end{multline*}

\subsection{The spaces $\mathcal{D}^{l}(\mathcal{O}_{\mathbb{K}})$}

For $l\geq1$, we set $G_{l}\left(  \mathcal{O}_{\mathbb{K}}\right)
=G_{l}:=\mathcal{O}_{\mathbb{K}}/\mathfrak{\wp}^{l}\mathcal{O}_{\mathbb{K}}$.
We use the following system of representatives for the elements of $G_{l}$:%
\[
\boldsymbol{i}=i_{0}+i_{1}\mathfrak{\wp}+\ldots+i_{l-1}\mathfrak{\wp}%
^{l-1},\text{ }%
\]
where the $i_{j}\in\left\{  0,1,\ldots,p-1\right\}  $. The points of $G_{l}$
are naturally organized in a finite rooted tree with $l$ levels.

We set $\mathcal{D}^{l}(\mathcal{O}_{\mathbb{K}})$ to be the $\mathbb{R}%
$-vector space of all test functions of the form%
\begin{equation}
\varphi\left(  x\right)  =\sum\limits_{\boldsymbol{i}\in G_{l}}\varphi\left(
\boldsymbol{i}\right)  \Omega\left(  p^{l}\left\vert x-\boldsymbol{i}%
\right\vert _{\mathbb{K}}\right)  \text{, \ }\varphi\left(  \boldsymbol{i}%
\right)  \in\mathbb{R}\text{,} \label{Eq_repre}%
\end{equation}
where $\Omega\left(  p^{l}\left\vert x-\boldsymbol{i}\right\vert _{\mathbb{K}%
}\right)  $ is the characteristic function of the ball $\boldsymbol{i}%
+\mathfrak{\wp}^{l}\mathcal{O}_{\mathbb{K}}$. The function $\varphi$ is
supported on $\mathcal{O}_{\mathbb{K}}$ and has parameter of constancy $l$.
The vector space $\mathcal{D}^{l}(\mathcal{O}_{\mathbb{K}})$ is spanned by the
basis%
\begin{equation}
\left\{  \Omega\left(  p^{l}\left\vert x-\boldsymbol{i}\right\vert
_{\mathbb{K}}\right)  \right\}  _{\boldsymbol{i}\in G_{l}}. \label{Basis}%
\end{equation}
The identification of $\varphi\in\mathcal{D}^{l}(\mathcal{O}_{\mathbb{K}})$
with the column vector $\left[  \varphi\left(  \boldsymbol{i}\right)  \right]
_{\boldsymbol{i}\in G_{l}}\in\mathbb{R}^{\#G_{l}}$, with $\#G_{l}=p^{lN}$,
gives rise to an isomorphism between $\mathcal{D}^{l}(\mathcal{O}_{\mathbb{K}%
})$ and $\mathbb{R}^{\#G_{l}}$ endowed with the norm
\[
\left\Vert \left[  \varphi\left(  \boldsymbol{i}\right)  \right]
_{\boldsymbol{i}\in G_{l}}\right\Vert =\max_{\boldsymbol{i}\in G_{l}%
}\left\vert \varphi\left(  \boldsymbol{i}\right)  \right\vert =\left\Vert
\varphi\right\Vert _{\infty}.
\]
Furthermore,
\begin{equation}
\mathcal{D}^{l}(\mathcal{O}_{\mathbb{K}})\hookrightarrow\mathcal{D}%
^{l+1}(\mathcal{O}_{\mathbb{K}})\hookrightarrow\mathcal{D}(\mathcal{O}%
_{\mathbb{K}}), \label{Embedding}%
\end{equation}
where $\hookrightarrow$ denotes a continuous embedding, and
\begin{equation}
\mathcal{D}(\mathcal{O}_{\mathbb{K}})=\cup_{l}\mathcal{D}^{l}(\mathcal{O}%
_{\mathbb{K}}). \label{flag-decomposition}%
\end{equation}

\begin{remark}
\label{Nota_1}(i) We warn the reader that through the article, we use $\left(
\mathbb{R}^{N},\left\vert \cdot\right\vert \right)  $, where $\left\vert
\cdot\right\vert $ is the standard norm, and also $\left(  \mathbb{R}%
^{N},\left\Vert \cdot\right\Vert \right)  \simeq\left(  \mathcal{D}%
^{l}(\mathcal{O}_{\mathbb{K}}),\left\Vert \cdot\right\Vert _{\infty}\right)  $.

(ii) A vector of the form $X=\left[  X\left(  \boldsymbol{i}\right)  \right]
_{\boldsymbol{i}\in G_{l}}\in\mathbb{R}^{p^{l}}$ defines naturally a unique a
test function from $\mathcal{D}^{l}(\mathcal{O}_{\mathbb{K}})$, more
precisely, $\varphi_{X}\left(  y\right)  =\sum_{\boldsymbol{i}\in G_{l}%
}X\left(  \boldsymbol{i}\right)  \Omega\left(  p^{l}\left\vert
x-\boldsymbol{i}\right\vert _{\mathbb{K}}\right)  $. Notice that $\varphi
_{X}\in L^{\rho}$, $\rho\in\left[  1,\infty\right)  $, and
\[
\left\Vert \varphi_{X}\right\Vert _{\rho}=p^{\frac{-1}{\rho}}\left(
\sum_{\boldsymbol{i}\in G_{l}}\left\vert X\left(  \boldsymbol{i}\right)
\right\vert ^{\rho}\right)  ^{\frac{1}{\rho}}.
\]

(iii) We also identify $X$ with a weighted rooted tree $(G_{l},W_{X})$, where
the weight of vertex $\boldsymbol{i}$ is given as $W_{X}\left(  \boldsymbol{i}%
\right)  =X\left(  \boldsymbol{i}\right)  $.
\end{remark}

\subsection{Positive characteristic versus \ zero characteristic}

For $l\geq2$, we denote by $\Lambda_{l}$ the group homomorphism%
\[%
\begin{array}
[c]{cccc}%
\Lambda_{l}: & G_{l} & \rightarrow & G_{l-1}\\
&  &  & \\
& a_{0}+a_{1}\mathfrak{\wp}+\ldots+a_{l-1}\mathfrak{\wp}^{l-1} & \rightarrow &
a_{0}+a_{1}\mathfrak{\wp}+\ldots+a_{l-2}\mathfrak{\wp}^{l-2}.
\end{array}
\]
For a given $\boldsymbol{j}\in G_{l-1}$, the elements in $\boldsymbol{k}\in
G_{l}$ such that $\Lambda_{l}\left(  \boldsymbol{k}\right)  =$ $\boldsymbol{j}%
$ are called the liftings of $\boldsymbol{j}$ to $G_{l}$. In positive
characteristic,
\[
G_{l}=\left\{  q(T)\in\mathbb{F}_{p}\left[  T\right]  ;q(T)=a_{0}%
+a_{1}T+\ldots+a_{l-1}T^{l-1}\right\}
\]
is the additive group of polynomials of degree up most $l-1$.\ Then, $G_{l-1}
$ is an additive subgroup of $G_{l}$. In characteristic zero,
\[
G_{l}=\left\{  a_{0}+a_{1}p+\ldots+a_{l-1}p^{l-1};a_{i}\in\left\{
0,\ldots,p-1\right\}  \right\}  .
\]
The additive group $G_{l-1}$ is not an additive subgroup of $G_{l}$.

\section{Discrete non-Archimedean DNNs}

\begin{lemma}
\label{Lemma1} Assume that $X\in\mathcal{D}^{l-1}\left(  \mathcal{O}%
_{\mathbb{K}}\right)  $, $w(\cdot,y)\in\mathcal{D}^{l}\left(  \mathcal{O}%
_{\mathbb{K}}\right)  $, $\theta\in\mathcal{D}^{l}\left(  \mathcal{O}%
_{\mathbb{K}}\right)  $, for some $l\geq2$. Then, all the continuous
non-Archimedean DNNs with parameters $\sigma_{M},w,\sigma,\theta$ belong to
$\mathcal{D}^{l}(\mathcal{O}_{\mathbb{K}})$ (i.e., $Y\in\mathcal{D}^{l}\left(
\mathcal{O}_{\mathbb{K}}\right)  $) and
\begin{equation}
Y\left(  \boldsymbol{i}\right)  =\sigma_{M}\left(
{\displaystyle\sum\limits_{\boldsymbol{k}\in G_{l}}}
w\left(  \boldsymbol{i},\boldsymbol{k}\right)  X\left(  \Lambda_{l}\left(
\boldsymbol{k}\right)  \right)  +\theta\left(  \boldsymbol{i}\right)  \right)
\text{, for }\boldsymbol{i}\in G_{l}\text{.} \label{Mapping}%
\end{equation}

\end{lemma}

\begin{proof}
The hypotheses $X\in\mathcal{D}^{l-1}\left(  \mathcal{O}_{\mathbb{K}}\right)
$, $w(\cdot,y)\in\mathcal{D}^{l}\left(  \mathcal{O}_{\mathbb{K}}\right)  $,
imply that%
\[
X(y)=%
{\displaystyle\sum\limits_{\boldsymbol{j}\in G_{l-1}}}
X\left(  \boldsymbol{j}\right)  \Omega\left(  p^{l-1}\left\vert
y-\boldsymbol{j}\right\vert _{\mathbb{K}}\right)  \text{, }w(x,y)=%
{\displaystyle\sum\limits_{\boldsymbol{j}\in G_{l}}}
w\left(  \boldsymbol{j},y\right)  \Omega\left(  p^{l}\left\vert
x-\boldsymbol{j}\right\vert _{p}\right)  .
\]
In addition,
\[
\theta(x)=%
{\displaystyle\sum\limits_{\boldsymbol{j}\in G_{l}}}
\theta\left(  \boldsymbol{j}\right)  \Omega\left(  p^{l}\left\vert
x-\boldsymbol{j}\right\vert _{p}\right)  .
\]
Now, by using the partition%
\[
\boldsymbol{j}+\mathfrak{\wp}^{l-1}\mathcal{O}_{\mathbb{K}}=%
{\displaystyle\bigsqcup\limits_{\boldsymbol{k}\in T_{l}^{N}}}
\left(  \boldsymbol{k}+\mathfrak{\wp}^{l}\mathcal{O}_{\mathbb{K}}\right)  ,
\]
where%
\[
T_{l}^{N}:=\left\{  \boldsymbol{k}\in G_{l};\Lambda_{l}\left(  \boldsymbol{k}%
\right)  \in\boldsymbol{j}\right\}  ,
\]
one gets that
\begin{gather*}
X\left(  \boldsymbol{j}\right)  \Omega\left(  p^{l-1}\left\vert
y-\boldsymbol{j}\right\vert _{\mathbb{K}}\right)  =%
{\displaystyle\sum\limits_{\substack{\boldsymbol{k}\in G_{l}\\\Lambda
_{l}\left(  \boldsymbol{k}\right)  =\boldsymbol{j}}}}
X\left(  \boldsymbol{j}\right)  \Omega\left(  p^{l}\left\vert y-\boldsymbol{k}%
\right\vert _{\mathbb{K}}\right) \\
=%
{\displaystyle\sum\limits_{\substack{\boldsymbol{k}\in G_{l}\\\Lambda
_{l}\left(  \boldsymbol{k}\right)  =\boldsymbol{j}}}}
X\left(  \Lambda_{l}\left(  \boldsymbol{k}\right)  \right)  \Omega\left(
p^{l}\left\vert y-\boldsymbol{k}\right\vert _{\mathbb{K}}\right)  ,
\end{gather*}
and%
\[
X(y)=%
{\displaystyle\sum\limits_{\boldsymbol{j}\in G_{l-1}}}
X\left(  \boldsymbol{j}\right)  \Omega\left(  p^{l-1}\left\vert
y-\boldsymbol{j}\right\vert _{\mathbb{K}}\right)  =%
{\displaystyle\sum\limits_{\boldsymbol{k}\in G_{l}}}
X\left(  \Lambda_{l}\left(  \boldsymbol{k}\right)  \right)  \Omega\left(
p^{l}\left\vert y-\boldsymbol{k}\right\vert _{\mathbb{K}}\right)  .
\]
Therefore%
\begin{multline*}
\int\limits_{\mathcal{O}_{\mathbb{K}}}w(x,y)X(y)dy\\
=%
{\displaystyle\sum\limits_{\boldsymbol{j}\in G_{l}}}
\left\{
{\displaystyle\sum\limits_{\boldsymbol{k}\in G_{l}}}
X\left(  \Lambda_{l}\left(  \boldsymbol{k}\right)  \right)  \int
\limits_{\mathcal{O}_{\mathbb{K}}}w\left(  \boldsymbol{j},y\right)
\Omega\left(  p^{l}\left\vert y-\boldsymbol{k}\right\vert _{\mathbb{K}%
}\right)  d^{N}y\right\}  \Omega\left(  p^{l}\left\vert x-\boldsymbol{j}%
\right\vert _{\mathbb{K}}\right) \\
=%
{\displaystyle\sum\limits_{\boldsymbol{j}\in G_{l}}}
\left\{
{\displaystyle\sum\limits_{\boldsymbol{k}\in G_{l}}}
X\left(  \Lambda_{l}\left(  \boldsymbol{k}\right)  \right)  w\left(
\boldsymbol{j},\boldsymbol{k}\right)  \right\}  \Omega\left(  p^{l}\left\vert
x-\boldsymbol{j}\right\vert _{p}\right)  ,
\end{multline*}
where%
\[
w\left(  \boldsymbol{j},\boldsymbol{k}\right)  :=\int\limits_{\mathcal{O}%
_{\mathbb{K}}}w\left(  \boldsymbol{j},y\right)  \Omega\left(  p^{l}\left\vert
y-\boldsymbol{k}\right\vert _{\mathbb{K}}\right)  dy.
\]
In conclusion,%
\begin{multline*}
\sigma\left(  \text{ }\int\limits_{\mathcal{O}_{\mathbb{K}}}w(x,y)X(y)d^{N}%
y+\theta\left(  x\right)  \right) \\
=%
{\displaystyle\sum\limits_{\boldsymbol{j}\in G_{l}}}
\sigma\left(
{\displaystyle\sum\limits_{\boldsymbol{k}\in G_{l}}}
w\left(  \boldsymbol{j},\boldsymbol{k}\right)  X\left(  \Lambda_{l}\left(
\boldsymbol{k}\right)  \right)  +\theta\left(  \boldsymbol{j}\right)  \right)
\Omega\left(  p^{l}\left\vert x-\boldsymbol{j}\right\vert _{\mathbb{K}%
}\right)  .
\end{multline*}

\end{proof}

\begin{definition}
\label{Definition_DNN}Given $\sigma_{M}$ an activation function, $L,\Delta$
positive integers,
\[
\theta\in\mathcal{D}^{L+\Delta}\left(  \mathcal{O}_{\mathbb{K}}\right)
\text{,}\ \ w\in\mathcal{D}^{L+\Delta}\left(  \mathcal{O}_{\mathbb{K}}%
\times\mathcal{O}_{\mathbb{K}}\right)  ,
\]
we define the non-Archimedean discrete $DNN(\mathcal{O}_{\mathbb{K}}%
,\sigma_{M},L,\Delta,w,\theta)$ as a NN with input%
\[
X^{in}\left(  \boldsymbol{j}\right)  =X^{\left(  L\right)  }\left(
\boldsymbol{j}\right)  \text{, for }\boldsymbol{j}\in G_{L}\text{,}%
\]
states%
\begin{equation}
X^{\left(  l\right)  }\left(  \boldsymbol{j}\right)  =\sigma_{M}\left(
{\displaystyle\sum\limits_{\boldsymbol{k}\in G_{l}}}
w\left(  \boldsymbol{j},\boldsymbol{k}\right)  X^{\left(  l-1\right)  }\left(
\Lambda_{l}\left(  \boldsymbol{k}\right)  \right)  +\theta\left(
\boldsymbol{j}\right)  \right)  \label{Eq_DBN}%
\end{equation}
for $l=L+1,\ldots,L+\Delta$, and $\boldsymbol{j}\in G_{l}$, and output%
\[
Y^{out}\left(  \boldsymbol{j}\right)  =X^{\left(  L+\Delta\right)  }\left(
\boldsymbol{j}\right)  \text{, for }\boldsymbol{j}\in G_{L+\Delta}\text{.}%
\]
We say that $DNN(\mathcal{O}_{\mathbb{K}},\sigma_{M},L,\Delta,w,\theta)$ is a
convolutional network, if the states are determined by%
\[
X^{\left(  l\right)  }\left(  \boldsymbol{j}\right)  =\sigma_{M}\left(
{\displaystyle\sum\limits_{\boldsymbol{k}\in G_{l}}}
w\left(  \boldsymbol{j}-\boldsymbol{k}\right)  X^{\left(  l-1\right)  }\left(
\Lambda_{l}\left(  \boldsymbol{k}\right)  \right)  +\theta\left(
\boldsymbol{j}\right)  \right)
\]
for $l=L+1,\ldots,L+\Delta$.
\end{definition}

The input and the output are test functions, more precisely, $X^{in}%
\in\mathcal{D}^{L}\left(  \mathcal{O}_{\mathbb{K}}\right)  $ $\ $and
$Y^{out}\in\mathcal{D}^{L+\Delta}\left(  \mathcal{O}_{\mathbb{K}}\right)  $.
The vector $X^{in}$ is the initial state of the network. The equation
(\ref{Eq_DBN} establishes that the weights and biases of layer $l$ and the
state of layer $l-1$ determine the state of the network at layer $l$. The
network neurons are organized in a tree-like structure with $\Delta$\ levels.
The parameter $\Delta$\ gives the depth of the network. At the level $l$, with
$L+1\leq l\leq L+\Delta-1$, there are $p^{l}$ neurons. The neuron
$\boldsymbol{j}\in G_{l}$ is connected with the neuron $\boldsymbol{k}\in
G_{l+1}$ if there exist $\boldsymbol{s}\in G_{l+1}$ such that $\Lambda
_{l+1}\left(  \boldsymbol{s}\right)  =\boldsymbol{j}$ $\ $and $w\left(
\boldsymbol{j},\boldsymbol{s}\right)  \neq0$. Since the cardinality of $G_{L}$
is $p^{L}$, the length of the input should be less than or equal $p^{L}$. For
practical purposes, we pick the smallest possible prime $p$, and then $L$ is
determined by the input size.

\section{Implementation of discrete non-Archimedean deep neural networks}

\subsection{Matrix form of the DNNs}

From now on, for the sake of simplicity, we use matrix notation. We set%
\begin{align*}
X^{\left[  l\right]  }  &  :=\left[  X\left(  \boldsymbol{k}\right)  \right]
_{\boldsymbol{k}\in G_{l}}=\left[  X_{\boldsymbol{k}}^{\left[  l\right]
}\right]  _{\boldsymbol{k}\in G_{l}}\in\mathbb{R}^{p^{l}}\text{, }\\
\underline{X}^{\left[  l-1\right]  }  &  :=\left[  X^{\left(  l-1\right)
}\left(  \Lambda_{l}\left(  \boldsymbol{k}\right)  \right)  \right]
_{\boldsymbol{k}\in G_{l}}=\left[  \underline{X}_{\boldsymbol{k}}^{\left[
l-1\right]  }\right]  _{\boldsymbol{k}\in G_{l}}\in\mathbb{R}^{p^{l}}\text{,
}\\
\theta^{\left[  l\right]  }  &  :=\left[  \theta\left(  \boldsymbol{k}\right)
\right]  _{\boldsymbol{k}\in G_{l}}=\left[  \theta_{\boldsymbol{k}}^{\left[
l\right]  }\right]  _{\boldsymbol{k}\in G_{l}}\in\mathbb{R}^{p^{l}},
\end{align*}
for $l=L+1,\ldots,L+\Delta$, and the matrix
\[
W^{\left[  l\right]  }:=\left[  w\left(  \boldsymbol{i},\boldsymbol{k}\right)
\right]  _{\boldsymbol{i},\boldsymbol{k}\in G_{l}}=\left[  w_{\boldsymbol{i}%
,\boldsymbol{k}}^{\left[  l\right]  }\right]  _{\boldsymbol{i},\boldsymbol{k}%
\in G_{l}}.
\]
Notice that $X^{\left[  l-1\right]  }\neq\underline{X}^{\left[  l-1\right]  }%
$. This is an important difference with the standard case. $\underline
{X}^{\left[  l-1\right]  }$ denotes a state at layer $l$ obtained by copying a
state $X^{\left[  l-1\right]  }$ at layer $l-1$.

With the above notation, a non-Archimedean DNN can be rewritten as
\begin{equation}
\left\{
\begin{array}
[c]{l}%
X^{\left[  L\right]  }\text{ input}\\
\\
X^{\left[  L+\Delta\right]  }\text{ \ output}\\
\\
Z^{\left[  l\right]  }:=W^{\left[  l\right]  }\underline{X}^{\left[
l-1\right]  }+\theta^{\left[  l\right]  }\\
\\
X^{\left[  l\right]  }=\sigma_{M}\left(  Z^{\left[  l\right]  }\right)
\text{, }l=L+1,\ldots,L+\Delta,
\end{array}
\right.  \label{Eq_MLP_Matrix}%
\end{equation}

where
\[
\left(  W^{\left[  l\right]  }\underline{X}^{\left[  l-1\right]  }%
+\theta^{\left[  l\right]  }\right)  _{\boldsymbol{j}}=\sigma_{M}\left(
{\displaystyle\sum\limits_{\boldsymbol{k}\in G_{l}}}
\underline{X}_{\boldsymbol{k}}^{\left[  l-1\right]  }w_{\boldsymbol{j}%
,\boldsymbol{k}}^{\left[  l\right]  }+\theta_{\boldsymbol{j}}^{\left[
l\right]  }\right)  \text{, for }\boldsymbol{j}\in G_{l},
\]
see (\ref{Eq_DNN}).

\subsection{Stochastic gradient}

\begin{notation}
Given two vectors $A=\left[  A_{\boldsymbol{k}}\right]  _{\boldsymbol{k}\in
G_{l}}$, $B=\left[  B_{\boldsymbol{k}}\right]  _{\boldsymbol{k}\in G_{l}}%
\in\mathbb{R}^{p^{l}}$, we set%
\[
\left\vert A-B\right\vert =\sqrt{%
{\displaystyle\sum\limits_{\boldsymbol{k}\in G_{l}}}
\left(  A_{\boldsymbol{k}}-B_{\boldsymbol{k}}\right)  ^{2}}.
\]

\end{notation}

Assume that $\left\{  A^{\left\{  i\right\}  }\right\}  _{i=1}^{K}$ are
training data for which there are given target outputs $\left\{  Y\left(
A^{\left\{  i\right\}  }\right)  \right\}  _{i=1}^{K}$. We use the quadratic
cost function%
\begin{equation}
C=\frac{1}{K}%
{\displaystyle\sum\limits_{i=1}^{K}}
\frac{1}{2}\left\vert Y\left(  A^{\left\{  i\right\}  }\right)  -X^{\left[
L+\Delta\right]  }\left(  A^{\left\{  i\right\}  }\right)  \right\vert
^{2}=\frac{1}{K}%
{\displaystyle\sum\limits_{i=1}^{K}}
C_{A^{\left\{  i\right\}  }}, \label{Cost_Function}%
\end{equation}
to determine the parameters of the network, where $X^{\left[  L+\Delta\right]
}\left(  A^{\left\{  i\right\}  }\right)  $ denotes the output of the network
with input $A^{\left\{  i\right\}  }$, and
\[
C_{A^{\left\{  i\right\}  }}=\left\vert Y\left(  A^{\left\{  i\right\}
}\right)  -X^{\left[  L+\Delta\right]  }\left(  A^{\left\{  i\right\}
}\right)  \right\vert ^{2}.
\]

We denote by $\Theta=\left(  W^{\left[  L+\Delta\right]  },\theta^{\left[
L+\Delta\right]  }\right)  \in\mathbb{R}^{s}$, with $s:=p^{2(L+\Delta
)}+p^{(L+\Delta)}$, the vector constructed with the weights $W^{\left[
L+\Delta\right]  }$ and the biases $\theta^{\left[  L+\Delta\right]  }$ of the
network, thus, $C=C(\Theta):\mathbb{R}^{s}\rightarrow\mathbb{R}$. The
minimization of the cost function is obtained by using the gradient descent
method:%
\begin{equation}
\Theta\rightarrow\Theta-\eta\nabla C(\Theta)=\Theta-\frac{\eta}{K}%
{\displaystyle\sum\limits_{i=1}^{K}}
\nabla C_{A^{\left\{  i\right\}  }}(\Theta), \label{Gradient_Approximation}%
\end{equation}
where $\eta$ is the learning constant and $\nabla C(\Theta)$ denotes the
gradient of the function $C(\Theta)$. A practical implementation of the
updating method (\ref{Gradient_Approximation}) is obtained by using the
stochastic gradient; see, e.g., \cite{Higham}, and the references therein.

It is relevant to note that cost function (\ref{Cost_Function}) is \ (up to a
constant) the $L^{2}$-norm of two functions from $\mathcal{D}^{L+\Delta
}\left(  \mathcal{O}_{\mathbb{K}}\right)  $. Thus, the network computes a
function $X^{\left[  L+\Delta\right]  }\in\mathcal{D}^{L+\Delta}\left(
\mathcal{O}_{\mathbb{K}}\right)  $ which is very closed (in the $L^{2}$-norm)
to a given function. Of course, in (\ref{Cost_Function}), we may replace $2$
by $\rho\in\left[  1,\infty\right)  $.

\subsection{Backpropagation}

We fix a training input $A^{\left\{  i\right\}  }$ and consider
$C_{A^{\left\{  i\right\}  }}(\Theta)$ as a function of $\Theta$, and use the
notation
\begin{equation}
C_{A^{\left\{  i\right\}  }}(\Theta):=C=\frac{1}{2}\left\vert Y-X^{\left[
L+\Delta\right]  }\right\vert ^{2}. \label{Eq_C_rest}%
\end{equation}
We set $Z^{\left[  l\right]  }:=\left[  Z_{\boldsymbol{k}}^{\left[  l\right]
}\right]  _{\boldsymbol{k}\in G_{l}}$, see (\ref{Eq_MLP_Matrix}). We refer
$Z_{\boldsymbol{k}}^{\left[  l\right]  }$ as the weighted input for neuron
$\boldsymbol{k}\in G_{l}$, at layer $l$. We also set
\[
\delta_{\boldsymbol{k}}^{\left[  l\right]  }=\frac{\partial C}{\partial
Z_{\boldsymbol{k}}^{\left[  l\right]  }}\text{ for }\boldsymbol{k}\in
G_{l}\text{ and }l=L+1,\ldots,L+\Delta\text{,}%
\]
and $\delta^{\left[  l\right]  }:=\left[  \delta_{\boldsymbol{k}}^{\left[
l\right]  }\right]  _{\boldsymbol{k}\in G_{l}}$.

Given two vectors $A=\left[  A_{\boldsymbol{k}}\right]  _{\boldsymbol{k}\in
G_{l}}$, $B=\left[  B_{\boldsymbol{k}}\right]  _{\boldsymbol{k}\in G_{l}}%
\in\mathbb{R}^{p^{lN}}$, the Hadamard product of them is defined as \
\[
A\circ B:=\left[  \left(  A\circ B\right)  _{\boldsymbol{k}}\right]
_{\boldsymbol{k}\in G_{l}}=\left[  A_{\boldsymbol{k}}B_{\boldsymbol{k}%
}\right]  _{\boldsymbol{k}\in G_{l}}.
\]
We also use the notation $\sigma_{M}\left(  Z^{\left[  l\right]  }\right)
:=\left[  \sigma_{M}\left(  Z_{\boldsymbol{k}}^{\left[  l\right]  }\right)
\right]  _{\boldsymbol{k}\in G_{l}}$, and
\[
\sigma_{M}^{\prime}\left(  Z^{\left[  l\right]  }\right)  =\frac{d}{dt}%
\sigma_{M}\left(  Z^{\left[  l\right]  }\right)  =\left[  \frac{d}{dt}%
\sigma_{M}\left(  Z_{\boldsymbol{k}}^{\left[  l\right]  }\right)  \right]
_{\boldsymbol{k}\in G_{l}}.
\]

\begin{lemma}
\label{Lemma_Key}With the above notation the following formulae hold:

\noindent(i) $\delta^{\left[  L+\Delta\right]  }=\sigma_{M}^{\prime}\left(
Z^{\left[  l\right]  }\right)  \circ\left(  X^{\left[  L+\Delta\right]
}-Y\right)  ;$

\noindent(ii) $\delta^{\left[  l\right]  }=\sigma_{M}^{\prime}\left(
Z^{\left[  l\right]  }\right)  \circ\left(  W^{\left[  l+1\right]  }\right)
^{T}\delta^{\left[  l+1\right]  }$ for $l=L+1,\ldots,L+\Delta-1;$

\noindent(iii) $\frac{\partial C}{\partial\theta_{\boldsymbol{k}}^{\left[
l\right]  }}=\delta_{\boldsymbol{k}}^{\left[  l\right]  }$ \ for
$l=L+1,\ldots,L+\Delta;$

\noindent(iv) $\frac{\partial C}{\partial w_{\boldsymbol{k},\boldsymbol{j}%
}^{\left[  l\right]  }}=\delta_{\boldsymbol{k}}^{\left[  l\right]  }%
\underline{X}_{\boldsymbol{j}}^{\left[  l-1\right]  }$ \ for $l=L+1,\ldots
,L+\Delta.$
\end{lemma}

\begin{proof}
The proof is a simple variation of the classical one. See, for instance, the
proof of Lemma 5.1 in \cite{Higham}.
\end{proof}

The output $X^{\left[  L+\Delta\right]  }$ can be computed from the input,
$X^{\left[  L\right]  }$, using
\[
Z^{\left[  l\right]  }=W^{\left[  l\right]  }\underline{X}^{\left[
l-1\right]  }+\theta^{\left[  l\right]  }\text{, }X^{\left[  l\right]
}\text{, }l=L+1,\ldots,L+\Delta\text{,}%
\]
by computing $X^{\left[  L\right]  }$, $Z^{\left[  L+1\right]  }$,
\ldots,$Z^{\left[  L+\Delta-1\right]  }$, $X^{\left[  L+\Delta\right]  }$ in
order. After this step, $\delta^{\left[  L+\Delta\right]  }$ can be computed,
see Lemma \ref{Lemma_Key}-(i). Then, by Lemma \ref{Lemma_Key}-(ii),
$\delta^{\left[  L+\Delta-1\right]  }$, $\delta^{\left[  L+\Delta-2\right]  }%
$, \ldots$,\delta^{\left[  L+1\right]  }$ can be computed \ in a backward
pass. Now from Lemma \ref{Lemma_Key}-(iii),(iv), we can compute the partial
derivatives, and thus the gradients. This a non-Archimedean version of the
standard backpropagation algorithm.

\section{The discrete non-Archimedean DNNs are robust universal approximators}

\begin{theorem}
\label{Theorem1}Given $f\in L^{\rho}(\mathcal{O}_{\mathbb{K}})$, $1\leq
\rho\leq\infty$, with $\left\Vert f\right\Vert _{\rho}<M$, and any
$\epsilon>0$; then, for any input $X\in\mathcal{D}^{L}\left(  \mathcal{O}%
_{\mathbb{K}}\right)  $, $L\geq1$, there exists a $p$-adic $DNN(\mathcal{O}%
_{\mathbb{K}},\sigma_{M},L,\Delta,w,\theta)$ such that the output
\[
Y(x;\mathcal{O}_{\mathbb{K}},\sigma_{M},L,\Delta,w,\theta)=%
{\displaystyle\sum\limits_{\boldsymbol{k}\in G_{L+\Delta}}}
Y\left(  \boldsymbol{k}\right)  \Omega\left(  p^{L+\Delta}\left\vert
x-\boldsymbol{k}\right\vert _{\mathbb{K}}\right)
\]
satisfies
\[
\left\Vert Y\left(  \mathcal{O}_{\mathbb{K}},\sigma_{M},L,\Delta
,w,\theta\right)  -f\right\Vert _{\rho}<\epsilon,
\]
for any $w$ in a ball in $\mathcal{D}^{L+\Delta}\left(  \mathcal{O}%
_{\mathbb{K}}\times\mathcal{O}_{\mathbb{K}}\right)  \approx\mathbb{R}%
^{2p^{\left(  L+\Delta\right)  }}$, and $\ $any $\theta$ in a ball in
$\mathcal{D}^{L+\Delta}\left(  \mathcal{O}_{\mathbb{K}}\right)  \approx
\mathbb{R}^{p^{\left(  L+\Delta\right)  }}$. Furthermore, this approximation
property remains valid if we assume that the output is produced by a
convolutional network.
\end{theorem}

\begin{remark}
In practical applications the DNNs are initialize assigning random values to
the weights and biases, and the initial state. The value of the initial does
not affect the performance of the network.
\end{remark}

\begin{proof}
First, we may assume that $f\in\mathcal{D}\left(  \mathcal{O}_{\mathbb{K}%
}\right)  $, with $\left\Vert f\right\Vert _{\rho}<M$. Indeed, since
$\mathcal{D}\left(  \mathcal{O}_{\mathbb{K}}\right)  $ is dense in $L^{\rho
}(\mathcal{O}_{\mathbb{K}})$, there exists $\varphi\in\mathcal{D}%
^{L^{^{\prime}}}\left(  \mathcal{O}_{\mathbb{K}}\right)  $ such that
$\left\Vert \varphi-f\right\Vert _{\rho}<\frac{\epsilon}{2}$. Since
$\mathcal{D}^{L^{\prime}}\left(  \mathcal{O}_{\mathbb{K}}\right)
\hookrightarrow\mathcal{D}^{L^{\prime\prime}}\left(  \mathcal{O}_{\mathbb{K}%
}\right)  $ for $L^{\prime\prime}>L^{\prime}$, without loss of generality, we
may suppose that $L^{^{\prime}}>L$ so $L^{^{\prime}}=L+\Delta$, for some
positive integer $\Delta$. Then, it is sufficient to construct an DNN with
output $Y\in\mathcal{D}^{L+\Delta}\left(  \mathcal{O}_{\mathbb{K}}\right)  $
such that
\begin{equation}
\left\Vert \varphi-Y\right\Vert _{\infty}<\frac{\epsilon}{2},
\label{Eq_condition}%
\end{equation}
since $\left\Vert \varphi-Y\right\Vert _{\rho}\leq\left\Vert \varphi
-Y\right\Vert _{\infty}$. So, we may replace $f$ by $\varphi$, with
$\left\Vert \varphi\right\Vert _{\rho}\leq\left\Vert \varphi\right\Vert
_{\infty}<M$. Taking
\[
\varphi\left(  x\right)  =%
{\displaystyle\sum\limits_{\boldsymbol{i}\in G_{L+\Delta}}}
\varphi\left(  \boldsymbol{i}\right)  \Omega\left(  p^{L+\Delta}\left\vert
x-\boldsymbol{i}\right\vert _{\mathbb{K}}\right)  ,
\]
we look for an DNN with output $Y\in\mathcal{D}^{L+\Delta}\left(
\mathcal{O}_{\mathbb{K}}\right)  $ such that%
\begin{equation}
\max_{\boldsymbol{i}\in G_{L+\Delta}}\left\vert Y\left(  \boldsymbol{i}%
\right)  -\varphi\left(  \boldsymbol{i}\right)  \right\vert <\epsilon\text{, }
\label{Approximation}%
\end{equation}
with $\varphi\left(  \boldsymbol{i}\right)  \in\left(  -M,M\right)  $. By
using the surjectivity of the activation function $\sigma_{M}$, for each
$\boldsymbol{i}\in G_{L+\Delta}^{N}$, we fix a preimage $\sigma_{M}%
^{-1}\left(  \varphi\left(  \boldsymbol{i}\right)  \right)  $, now by the
continuity of $\sigma_{M}$, there exists $\delta\left(  \epsilon\right)  $
such that
\[
\max_{\boldsymbol{i}\in G_{L+\Delta}}\left\vert
{\displaystyle\sum\limits_{\boldsymbol{k}\in G_{L+\Delta}^{N}}}
X\left(  \Lambda_{L+\Delta}\left(  \boldsymbol{k}\right)  \right)  w\left(
\boldsymbol{i},\boldsymbol{k}\right)  +\theta\left(  \boldsymbol{i}\right)
-\sigma_{M}^{-1}\left(  \varphi\left(  \boldsymbol{i}\right)  \right)
\right\vert <\delta\left(  \epsilon\right)
\]
implies (\ref{Approximation}). By taking $\left\vert \theta\left(
\boldsymbol{i}\right)  -\sigma_{M}^{-1}\left(  \varphi\left(  \boldsymbol{i}%
\right)  \right)  \right\vert <\frac{\delta\left(  \epsilon\right)  }{2}$, for
all $\boldsymbol{i}\in G_{L+\Delta}$, i.e.,
\begin{equation}
\max_{\boldsymbol{i}\in G_{L+\Delta}}\left\vert \theta\left(  \boldsymbol{i}%
\right)  -\sigma_{M}^{-1}\left(  \varphi\left(  \boldsymbol{i}\right)
\right)  \right\vert <\frac{\delta\left(  \epsilon\right)  }{2}\text{,}
\label{Condition_A}%
\end{equation}
it is sufficient to show that%
\begin{equation}
\max_{\boldsymbol{i}\in G_{L+\Delta}}\left\vert
{\displaystyle\sum\limits_{\boldsymbol{k}\in G_{L+\Delta}}}
X\left(  \Lambda_{L+\Delta}\left(  \boldsymbol{k}\right)  \right)  w\left(
\boldsymbol{i},\boldsymbol{k}\right)  \right\vert <\frac{\delta\left(
\epsilon\right)  }{2}\text{.} \label{Approximation_2}%
\end{equation}
For a fix $\boldsymbol{i}\in G_{L+\Delta}^{N}$, the linear mapping
\[%
\begin{array}
[c]{cccc}%
\mathbb{R}^{p^{\left(  L+\Delta\right)  }} & \rightarrow & \mathbb{R} & \\
&  &  & \\
w\left(  \boldsymbol{i},\boldsymbol{k}\right)  & \rightarrow &
{\displaystyle\sum\limits_{\boldsymbol{k}\in G_{L+\Delta}}}
X\left(  \Lambda_{L+\Delta}\left(  \boldsymbol{k}\right)  \right)  w\left(
\boldsymbol{i},\boldsymbol{k}\right)  &
\end{array}
\]
is continuous at the origin, even if $X\left(  \Lambda_{L+\Delta}\left(
\boldsymbol{k}\right)  \right)  =0$ for all $\boldsymbol{k}\in G_{L+\Delta}$,
then there exists $\gamma\left(  \boldsymbol{i},\delta\right)  $ such that%
\[
\max_{\boldsymbol{k}\in G_{L+\Delta}}\left\vert w\left(  \boldsymbol{i}%
,\boldsymbol{k}\right)  \right\vert <\gamma\left(  \boldsymbol{i}%
,\delta\right)  \text{ implies }\left\vert
{\displaystyle\sum\limits_{\boldsymbol{k}\in G_{L+\Delta}}}
X\left(  \Lambda_{L+\Delta}\left(  \boldsymbol{k}\right)  \right)  w\left(
\boldsymbol{i},\boldsymbol{k}\right)  \right\vert <\frac{\delta\left(
\epsilon\right)  }{2}\text{.}%
\]
Consequently,%
\begin{equation}
\max_{\boldsymbol{i}\in G_{L+\Delta}}\max_{\boldsymbol{k}\in G_{L+\Delta}%
}\left\vert w\left(  \boldsymbol{i},\boldsymbol{k}\right)  \right\vert
<\min_{\boldsymbol{i}\in G_{L+\Delta}}\gamma\left(  \boldsymbol{i}%
,\delta\right)  \label{Condition_B}%
\end{equation}
implies (\ref{Approximation_2}).

By the isomorphism of Banach spaces $\left(  \mathcal{D}^{L+\Delta}\left(
\mathcal{O}_{\mathbb{K}}\right)  ,\left\Vert \cdot\right\Vert _{\infty
}\right)  \simeq\left(  \mathbb{R}^{p^{L+\Delta}},\left\Vert \cdot\right\Vert
_{\infty}\right)  $, condition (\ref{Condition_A}) defines a ball in
$\mathbb{R}^{p^{\left(  L+\Delta\right)  }}$, while condition
(\ref{Condition_B}) defines defines a ball on $\mathbb{R}^{2p^{\left(
L+\Delta\right)  }}$, where the approximation (\ref{Approximation}) is valid.
Finally,\ we observe that the above reasoning is valid if we take $w\left(
\boldsymbol{i},\boldsymbol{k}\right)  =w\left(  \boldsymbol{i}-\boldsymbol{k}%
\right)  $.

The prime number $p$ and the integer $L$ are determined by the input
$X\in\mathcal{D}^{L}\left(  \mathcal{O}_{\mathbb{K}}\right)  \simeq
\mathbb{R}^{p^{L}}$. Then, $p^{L}$ should be larger than the length of the
vector $X$. Consequently, we may pick $p$ to be any arbitrary prime.
\end{proof}

A good approximation of a function $f\in L^{\rho}(\mathcal{O}_{\mathbb{K}})$
by a test function $\varphi\in\mathcal{D}^{L^{^{\prime}}}\left(
\mathcal{O}_{\mathbb{K}}\right)  $ requires a large $L^{^{\prime}}$, and
consequently, the computational power of the DNNs increase with $\Delta$. The
$p$-adic discrete DNNs are particular cases of the standard deep neural
networks; see, e.g.,\cite{Higham}. DNNs can be trained using backpropagation,
and their convolutional implementation significantly reduces the number of weights.

\begin{remark}
The field $\mathbb{K}$ is invariant under transformations of the form
$x\rightarrow b+\mathfrak{\wp}^{s}x$, with $b\in\mathbb{K}$, $s\in\mathbb{Z}$.
For this reason, we will consider that all networks $DNN(\mathcal{O}%
_{\mathbb{K}},L,\Delta,w,\theta)$ with outputs of the form $Y(b+\mathfrak{\wp
}^{s}x)$ are equivalent. In practical terms, to find an approximation of a
function $f:b+\mathfrak{\wp}^{s}\mathcal{O}_{\mathbb{K}}\rightarrow\mathbb{R}$
is the same as find an approximation of the function $f^{\ast}:\mathcal{O}%
_{\mathbb{K}}\rightarrow\mathbb{R}$, where $f^{\ast}\left(  x\right)
=f\left(  b+\mathfrak{\wp}^{s}x\right)  $.
\end{remark}

\begin{remark}
\label{Remark_Cybenko}The natural non-Archimedean version of the perceptron is%
\begin{equation}
G\left(  \left[  X_{\boldsymbol{k}}\right]  _{\boldsymbol{k}\in G_{l}}\right)
=\sigma\left(
{\displaystyle\sum\limits_{\boldsymbol{k}\in G_{l}}}
W_{\boldsymbol{k}}X_{\boldsymbol{k}}+b\right)  , \label{Perceptron_classical}%
\end{equation}
where $X=\left[  X_{\boldsymbol{k}}\right]  _{\boldsymbol{k}\in G_{l}%
},W=\left[  W_{\boldsymbol{k}}\right]  _{\boldsymbol{k}\in G_{l}}\in
\mathbb{R}^{p^{l}}$, and $b\in\mathbb{R}$. Here, the activation function
$\sigma$ is defined as in \cite[Definition 2.2.1]{Calin}. By using the
standard inner product of $\mathbb{R}^{p^{l}}$, we rewrite
(\ref{Perceptron_classical}) as $G(X)=\sigma\left(  W\cdot X+b\right)  $.
However, (\ref{Perceptron_classical}) is radically different from the
classical perceptron $\sigma\left(  w\cdot x+b\right)  $, where $w=(w_{1}%
,\ldots,w_{M})$, $x=(x_{1},\ldots,x_{M}),\in\mathbb{R}^{M}$, because $i$,
$x_{i}$ are real numbers, while $X_{\boldsymbol{k}}$ is a real number, but
$\boldsymbol{k}\in\mathbb{K}$.\ Cybenko's theorem, see \cite{Cybenko},
\cite[Theorem 9.3.6]{Calin}, asserts that the finite sums of the form%
\[
G(X)=%
{\displaystyle\sum\limits_{r=1}^{R}}
\alpha_{r}\sigma\left(  W_{r}\cdot X+b_{r}\right)
\]
are dense in the space of continuos functions defined on $\left[  0,1\right]
^{M}$ endowed with the norm $\left\Vert \cdot\right\Vert _{\infty}$. Besides
of the apparently similitude, the non-Archimedean DNNs considered here are a
new type of neural networks.
\end{remark}

\subsection{Parallelization of non-Archimedean DNNs}

A natural problem is the simultaneous approximations of several functions by
DNNs. More precisely, given $f_{i}\in L^{\rho}(\mathcal{O}_{\mathbb{K}})$,
$i=1,\ldots,R$, by applying Theorem \ref{Theorem1}, there are DNNs,
\[
DNN(\mathcal{O}_{\mathbb{K}},\sigma_{M_{i}},L_{i},\Delta_{i},w_{i},\theta
_{i})\text{, \ }i=1,\ldots,R,
\]
whose outputs satisfy
\[
\left\Vert Y_{i}\left(  \mathcal{O}_{\mathbb{K}},\sigma_{M_{i}},L_{i}%
,\Delta_{i},w_{i},\theta_{i}\right)  -f_{i}\right\Vert _{\rho}<\epsilon\text{,
},i=1,\ldots,R.
\]
Without loss of generality, we may assume that the parameter $p$ is the same
for all the DNNs. For general $f_{i}$, $i=1,\ldots,R$, it is not possible to
find a unique DNN that approximates all the $f_{i}$, but we can use the
$DNN(\mathcal{O}_{\mathbb{K}},\sigma_{M_{i}},L_{i},\Delta_{i},w_{i},\theta
_{i})$, $i=1,\ldots,R$, in parallel to construct a simultaneous approximation
of all the $f_{i}$.

We set
\begin{align*}
\boldsymbol{f}  &  =\left(  f_{1},\ldots,f_{R}\right)  \in L^{\rho
}(\mathcal{O}_{\mathbb{K}})%
{\textstyle\bigoplus}
\cdots%
{\textstyle\bigoplus}
L^{\rho}(\mathcal{O}_{\mathbb{K}})=\left(  L^{\rho}(\mathcal{O}_{\mathbb{K}%
})\right)  ^{R},\\
\boldsymbol{X}  &  =\left(  X_{1},\ldots,X_{R}\right)  \in%
{\displaystyle\bigoplus\limits_{i=1}^{R}}
\mathcal{D}^{L_{i}}\left(  \mathcal{O}_{\mathbb{K}}\right)  \text{,
}\boldsymbol{Y}=\left(  Y_{1},\ldots,Y_{R}\right)  \in%
{\displaystyle\bigoplus\limits_{i=1}^{R}}
\mathcal{D}^{L_{i}+\Delta_{i}}\left(  \mathcal{O}_{\mathbb{K}}\right)  ,\\
\boldsymbol{L}  &  =\left(  L_{1},\ldots,L_{R}\right)  \text{, }%
\boldsymbol{\Delta}=\left(  \Delta_{1},\ldots,\Delta_{R}\right)  \in
\mathbb{N}^{R},\\
\boldsymbol{w}  &  =\left(  w_{1},\ldots,w_{R}\right)  \in%
{\displaystyle\bigoplus\limits_{i=1}^{R}}
\mathcal{D}^{L_{i}+\Delta_{i}}\left(  \mathcal{O}_{\mathbb{K}}\times
\mathcal{O}_{\mathbb{K}}\right)  \text{, }\boldsymbol{\theta}=\left(
\theta_{1},\ldots,\theta_{R}\right)  \in%
{\displaystyle\bigoplus\limits_{i=1}^{R}}
\mathcal{D}^{L_{i}+\Delta_{i}}\left(  \mathcal{O}_{\mathbb{K}}\right)  ,\\
\boldsymbol{\sigma}  &  =\left(  \sigma_{M_{1}},\ldots,\sigma_{M_{R}}\right)
.
\end{align*}
Given $\boldsymbol{f}\in\left(  L^{\rho}(\mathcal{O}_{\mathbb{K}})\right)
^{R} $, we define the norm%
\[
\left\Vert \boldsymbol{f}\right\Vert :=\max_{1\leq i\leq R}\left\Vert
f_{i}\right\Vert _{\rho},
\]
and for $\boldsymbol{X}\in%
{\displaystyle\bigoplus\limits_{i=1}^{R}}
\mathcal{D}^{L_{i}}\left(  \mathcal{O}_{\mathbb{K}}\right)  $,%
\[
\left\Vert \boldsymbol{X}\right\Vert =\max_{1\leq i\leq R}\left\Vert
X_{i}\right\Vert _{\infty}.
\]

\begin{definition}
The direct product of the DNNs, $DNN(\mathcal{O}_{\mathbb{K}},\sigma_{M_{i}%
},L_{i},\Delta_{i},w_{i},\theta_{i})$, $i=1,\ldots,R$, is \ a DNN, denoted as
$DNN(\mathcal{O}_{\mathbb{K}},\boldsymbol{\sigma},R,\boldsymbol{L}%
,\boldsymbol{\Delta},\boldsymbol{w},\boldsymbol{\theta})$, with input
$\boldsymbol{X}\in%
{\displaystyle\bigoplus\limits_{i=1}^{R}}
\mathcal{D}^{L_{i}}\left(  \mathcal{O}_{\mathbb{K}}\right)  $, and output
$\boldsymbol{Y}\in%
{\displaystyle\bigoplus\limits_{i=1}^{R}}
\mathcal{D}^{L_{i}+\Delta_{i}}\left(  \mathcal{O}_{\mathbb{K}}\right)  $.
\end{definition}

By using the above definition and Theorem \ref{Theorem1}, we obtain the
following result.

\begin{corollary}
\label{Corollary2}Given any $\boldsymbol{f}\in\left(  L^{\rho}(\mathcal{O}%
_{\mathbb{K}})\right)  ^{R}$, with $\left\Vert \boldsymbol{f}\right\Vert
_{\rho}<M$, and any $\epsilon>0$; then, for any input $\boldsymbol{X}\in%
{\displaystyle\bigoplus\limits_{i=1}^{R}}
\mathcal{D}^{L_{i}}\left(  \mathcal{O}_{\mathbb{K}}\right)  $, with $L_{i}%
\geq1$ for all $i$,, there exists a $p$-adic $DNN(\mathcal{O}_{\mathbb{K}%
},R,\boldsymbol{L},\boldsymbol{\Delta},\boldsymbol{w},\boldsymbol{\theta})$
such that the output $\boldsymbol{Y}(x;p,\boldsymbol{\sigma},R,\boldsymbol{L}%
,\boldsymbol{\Delta},\boldsymbol{w},\boldsymbol{\theta})\in%
{\displaystyle\bigoplus\limits_{i=1}^{R}}
\mathcal{D}^{L_{i}+\Delta_{i}}\left(  \mathcal{O}_{\mathbb{K}}\right)  $
satisfies $\left\Vert \boldsymbol{Y}(x;\mathcal{O}_{\mathbb{K}}%
,\boldsymbol{\sigma},R,\boldsymbol{L},\boldsymbol{\Delta},\boldsymbol{w}%
,\boldsymbol{\theta})-\boldsymbol{f}\right\Vert _{\rho}<\epsilon$, for any
$\boldsymbol{w}$ in a ball in $%
{\displaystyle\bigoplus\limits_{i=1}^{R}}
\mathcal{D}^{L_{i}+\Delta_{i}}\left(  \mathcal{O}_{\mathbb{K}}\times
\mathcal{O}_{\mathbb{K}}\right)  $, and $\ $any $\boldsymbol{\theta}$ in a
ball in $%
{\displaystyle\bigoplus\limits_{i=1}^{R}}
\mathcal{D}^{L_{i}+\Delta_{i}}\left(  \mathcal{O}_{\mathbb{K}}\right)  $.
\end{corollary}

\subsection{Approximation of functions on open compact subsets}

We now extend Theorem \ref{Theorem1} to functions supported in an open compact
subset $\mathcal{K}$. Such a set is a disjoint union of balls, then, without
loss of generality, we assume that%
\[
\mathcal{K}=%
{\displaystyle\bigsqcup\limits_{i=1}^{R}}
\left(  a_{i}+\mathfrak{\wp}^{N_{i}}\mathcal{O}_{\mathbb{K}}\right)  ,
\]
where $a_{i}\in\mathbb{K}$, $N_{i}\in\mathbb{Z}$, for $i=1,\ldots,R$. Any
function $f:\mathcal{K}\rightarrow\mathbb{R}$ is completely determined by its
restrictions to the balls $B_{-N_{i}}(a_{i})=a_{i}+\mathfrak{\wp}^{N_{i}%
}\mathcal{O}_{\mathbb{K}}$. More precisely,%
\begin{equation}
f(x)=%
{\displaystyle\sum\limits_{i=1}^{R}}
f_{i}(x)\text{, with }f_{i}(x)=f(x)\Omega\left(  p^{N_{i}}\left\vert
x-a_{i}\right\vert _{\mathbb{K}}\right)  . \label{Decomposition}%
\end{equation}

We now define the linear transformation%
\[%
\begin{array}
[c]{cccc}%
T_{a_{i},N_{i}}: & a_{i}+\mathfrak{\wp}^{N_{i}}\mathcal{O}_{\mathbb{K}} &
\rightarrow & \mathcal{O}_{\mathbb{K}}\\
&  &  & \\
& x & \rightarrow & T_{a_{i},N_{i}}\left(  x\right)  =\mathfrak{\wp}^{-N_{i}%
}\left(  x-a_{i}\right)  .
\end{array}
\]
Then, the mapping%
\[%
\begin{array}
[c]{cccc}%
T_{a_{i},N_{i}}^{\ast}: & L^{\rho}\left(  a_{i}+\mathfrak{\wp}^{N_{i}%
}\mathcal{O}_{\mathbb{K}}\right)  & \rightarrow & L^{\rho}\left(
\mathcal{O}_{\mathbb{K}}\right) \\
&  &  & \\
& f & \rightarrow & T_{a_{i},N_{i}}^{\ast}f,
\end{array}
\]
with $T_{a_{i},N_{i}}^{\ast}f\left(  x\right)  :=f\left(  T_{a_{i},N_{i}}%
^{-1}\left(  x\right)  \right)  $, defines a linear bounded operator
satisfying%
\[
\left\Vert T_{a_{i},N_{i}}^{\ast}f\right\Vert _{\rho}=\left\{
\begin{array}
[c]{ccc}%
p^{\frac{N_{i}}{\rho}}\left\Vert f\right\Vert _{\rho} & \text{if } & \rho
\in\left[  1,\infty\right) \\
&  & \\
\left\Vert f\right\Vert _{\rho} & \text{if } & \rho=\infty.
\end{array}
\right.  .
\]
Consequently,
\[
L^{\rho}\left(  a_{i}+\mathfrak{\wp}^{N_{i}}\mathcal{O}_{\mathbb{K}}\right)
\hookrightarrow L^{\rho}\left(  \mathcal{O}_{\mathbb{K}}\right)  ,\text{
}1\leq\rho\leq\infty,
\]
are continuous embeddings.

Since $\mathcal{K}$ has finite Haar measure, $L^{\rho}(\mathcal{K}%
)\hookrightarrow L^{1}(\mathcal{K})$, $1\leq\rho\leq\infty,$ and
$\mathcal{D}(\mathcal{K})$, the space of test functions supported in
$\mathcal{K}$, is dense in $L^{\rho}(\mathcal{K})$, $1\leq\rho\leq\infty$. By
(\ref{Decomposition}),%
\[
L^{\rho}(\mathcal{K})=%
{\displaystyle\bigoplus\limits_{i=1}^{R}}
L^{\rho}(a_{i}+\mathfrak{\wp}^{N_{i}}\mathcal{O}_{\mathbb{K}}).
\]
When identifying $f\in L^{\rho}(\mathcal{K})$ with $\left(  f_{1},\ldots
,f_{R}\right)  $, $f_{i}\in L^{\rho}(a_{i}+\mathfrak{\wp}^{N_{i}}%
\mathcal{O}_{\mathbb{K}})$, we use the notation $\boldsymbol{f}=\left(
f_{1},\ldots,f_{R}\right)  $. Notice that%
\begin{align*}
\left\Vert \boldsymbol{f}\right\Vert _{\rho}  &  =\left(
{\displaystyle\sum\limits_{i=1}^{R}}
\text{ }%
{\displaystyle\int\limits_{a_{i}+\mathfrak{\wp}^{N_{i}}\mathcal{O}%
_{\mathbb{K}}}}
\left\vert f_{i}\left(  x\right)  \right\vert ^{\rho}dx\right)  ^{\frac
{1}{\rho}}\text{, }\rho\in\left[  1,\infty\right)  ,\\
\left\Vert \boldsymbol{f}\right\Vert _{\infty}  &  =\max_{1\leq i\leq
R}\left\{  \sup_{x\in a_{i}+\mathfrak{\wp}^{N_{i}}\mathcal{O}_{\mathbb{K}}%
}\left\vert f_{i}\left(  x\right)  \right\vert \right\}  .
\end{align*}

\begin{theorem}
\label{Theorem1A}Given $\boldsymbol{f}\in L^{\rho}(\mathcal{K})$, $1\leq
\rho\leq\infty$, with $\left\Vert \boldsymbol{f}\right\Vert _{\rho}<M$,\ and
$\epsilon>0$; then, for any input%
\[
\boldsymbol{X}=\left(  X_{1},\ldots,X_{R}\right)  \in%
{\displaystyle\bigoplus\limits_{i=1}^{R}}
\mathcal{D}^{L_{i}}\left(  a_{i}+\mathfrak{\wp}^{N_{i}}\mathcal{O}%
_{\mathbb{K}}\right)  ,
\]
there exits a $DNN(\mathcal{O}_{\mathbb{K}},\boldsymbol{\sigma}%
,R,\boldsymbol{L},\boldsymbol{\Delta},\boldsymbol{w},\boldsymbol{\theta})$
with output
\[
\boldsymbol{Y}=\left(  \left(  T_{a_{1},N_{1}}^{-1}\right)  ^{\ast}%
Y_{1},\ldots,\left(  T_{a_{R},N_{R}}^{-1}\right)  ^{\ast}Y_{R}\right)  \in%
{\displaystyle\bigoplus\limits_{i=1}^{R}}
\mathcal{D}^{L_{i}+\Delta_{i}}\left(  a_{i}+\mathfrak{\wp}^{N_{i}}%
\mathcal{O}_{\mathbb{K}}\right)  ,
\]
such that $\left\Vert \boldsymbol{f}-\boldsymbol{Y}\right\Vert _{\rho
}<\epsilon$, for $\boldsymbol{w}$, $\boldsymbol{\theta}$, running in certain balls.
\end{theorem}

\begin{proof}
We first construct a DNN to approximate $f_{i}\in L^{\rho}(a_{i}%
+\mathfrak{\wp}^{N_{i}}\mathcal{O}_{\mathbb{K}})$. Notice that $T_{a_{i}%
,N_{i}}^{\ast}f_{i}\in L^{\rho}\left(  \mathcal{O}_{\mathbb{K}}\right)  $,
with $\left\Vert T_{a_{i},N_{i}}^{\ast}f_{i}\right\Vert _{\rho}<M_{i}\left(
\rho\right)  $, where
\[
M_{i}\left(  \rho\right)  =\left\{
\begin{array}
[c]{ccc}%
p^{\frac{N_{i}}{\rho}}M & \text{if } & 1\leq\rho<\infty\\
&  & \\
M & \text{if} & \rho=\infty.
\end{array}
\right.
\]
By Theorem \ref{Theorem1}, there exists a%
\begin{equation}
DNN(\mathcal{O}_{\mathbb{K}},\sigma_{M_{i}\left(  \rho\right)  },L_{i}%
,\Delta_{i},w_{i},\theta) \label{Eq_DNN}%
\end{equation}
such that the output $Y_{i}(x;\mathcal{O}_{\mathbb{K}},\sigma_{M_{i}\left(
\rho\right)  },L,\Delta,w,\theta)$ satisfies%
\begin{equation}
\left\Vert Y_{i}\left(  \mathcal{O}_{\mathbb{K}},\sigma_{M_{i}\left(
\rho\right)  },L,\Delta,w,\theta_{i}\right)  -T_{a_{i},N_{i}}^{\ast}%
f_{i}\right\Vert _{\rho}<\frac{\epsilon}{\gamma_{i}\left(  \rho\right)  },
\label{Eq_Inequality}%
\end{equation}
with
\[
\gamma_{i}\left(  \rho\right)  =\left\{
\begin{array}
[c]{ccc}%
p^{\frac{N_{i}}{\rho}} & \text{if } & 1\leq\rho<\infty\\
&  & \\
1 & \text{if} & \rho=\infty,
\end{array}
\right.
\]
for any $w_{i}$ in a ball in $\mathbb{R}^{2p^{\left(  L_{i}+\Delta_{i}\right)
}}$, and$\ $any $\theta$ in a ball in $\mathbb{R}^{p^{\left(  L_{i}+\Delta
_{i}\right)  }}$.

The inequality (\ref{Eq_Inequality})\ means that%
\[
E_{i}(x):=Y_{i}\left(  y;\mathcal{O}_{\mathbb{K}},\sigma_{M_{i}\left(
\rho\right)  },L,\Delta,w,\theta\right)  -T_{a_{i},N_{i}}^{\ast}f_{i}\left(
x\right)  \text{, with }\left\Vert E_{i}\right\Vert _{\rho}<\frac{\epsilon
}{\gamma_{i}\left(  \rho\right)  },
\]
where the parameters $\sigma_{M_{i}\left(  \rho\right)  },L_{i},\Delta
_{i},w_{i},\theta_{i}$ are functions of $a_{i},N_{i}$. Then%
\[
\left(  T_{a_{i},N_{i}}^{-1}\right)  ^{\ast}E_{i}(x)=\left(  T_{a_{i},N_{i}%
}^{-1}\right)  ^{\ast}Y_{i}\left(  x;\mathcal{O}_{\mathbb{K}},L,\Delta
,w,\theta\right)  -f_{i}\left(  x\right)  .
\]
Now, for $1\leq\rho<\infty$, taking $y=p^{-N_{i}}\left(  x-a_{i}\right)  $,
$dy=p^{N_{i}}dx$, we have%
\begin{align*}
\left\Vert \left(  T_{a_{i},N_{i}}^{-1}\right)  ^{\ast}E_{i}\right\Vert
_{\rho}  &  =\left(
{\displaystyle\int\limits_{\mathcal{O}_{\mathbb{K}}}}
\left\vert \left(  T_{a_{i},N_{i}}^{-1}\right)  ^{\ast}E_{i}(x)\right\vert
^{\rho}dx\right)  ^{\frac{1}{\rho}}=\left(
{\displaystyle\int\limits_{\mathcal{O}_{\mathbb{K}}}}
\left\vert E_{i}(a_{i}+\mathfrak{\wp}^{N_{i}}x)\right\vert ^{\rho}dx\right)
^{\frac{1}{\rho}}\\
&  =p^{\frac{L_{i}}{\rho}}\left(
{\displaystyle\int\limits_{a_{i}+\mathfrak{\wp}^{N_{i}}\mathcal{O}%
_{\mathbb{K}}}}
\left\vert E_{i}(z)\right\vert ^{\rho}dz\right)  ^{\frac{1}{\rho}}<\epsilon.
\end{align*}
In the case $\rho=\infty$, $\left\Vert \left(  T_{a_{i},N_{i}}^{\ast}\right)
^{-1}E\right\Vert _{\rho}=\left\Vert E\right\Vert _{\rho}<\epsilon$. In
conclusion, there exits a $DNN(\mathcal{O}_{\mathbb{K}},\sigma_{M_{i}\left(
\rho\right)  },L_{i},\Delta_{i},w_{i},\theta)$, with output satisfying
\[
\left\Vert \left(  T_{a_{i},N_{i}}^{-1}\right)  ^{\ast}Y\left(  \mathcal{O}%
_{\mathbb{K}},\sigma_{M_{i}\left(  \rho\right)  },L_{i},\Delta_{i}%
,w_{i},\theta_{i}\right)  -f_{i}\right\Vert _{\rho}<\epsilon,
\]
for any $w_{i}$ in a ball in $\mathbb{R}^{2p^{\left(  L_{i}+\Delta_{i}\right)
}}$, and $\ $any $\theta$ in a ball in $\mathbb{R}^{p^{\left(  L_{i}%
+\Delta_{i}\right)  }}$. We now set
\begin{align*}
\boldsymbol{L}  &  =\left(  L_{1},\ldots,L_{R}\right)  \text{, }%
\boldsymbol{\Delta}=\left(  \Delta_{1},\ldots,\Delta_{R}\right)  \in
\mathbb{N}^{R},\\
\boldsymbol{w}  &  =\left(  w_{1},\ldots,w_{R}\right)  \in%
{\displaystyle\bigoplus\limits_{i=1}^{R}}
\mathcal{D}^{L_{i}+\Delta_{i}}\left(  \left(  a_{i}+\mathfrak{\wp}^{N_{i}%
}\mathcal{O}_{\mathbb{K}}\right)  \times\left(  a_{i}+\mathfrak{\wp}^{N_{i}%
}\mathcal{O}_{\mathbb{K}}\right)  \right)  \text{, }\\
\boldsymbol{\theta}  &  =\left(  \theta_{1},\ldots,\theta_{R}\right)  \in%
{\displaystyle\bigoplus\limits_{i=1}^{R}}
\mathcal{D}^{L_{i}+\Delta_{i}}\left(  a_{i}+\mathfrak{\wp}^{N_{i}}%
\mathcal{O}_{\mathbb{K}}\right)  ,\text{ \ }\boldsymbol{\sigma}=\left(
\sigma_{M_{1}\left(  \rho\right)  },\ldots,\sigma_{M_{R}\left(  \rho\right)
}\right)  .
\end{align*}
Then, by taking the direct product of (\ref{Eq_DNN}), we obtain a neural
network
\[
DNN(\mathcal{O}_{\mathbb{K}},\boldsymbol{\sigma,}R,\boldsymbol{L}%
,\boldsymbol{\Delta},\boldsymbol{w},\boldsymbol{\theta}),
\]
with output $\boldsymbol{Y}=\left(  \left(  T_{a_{1},N_{1}}^{-1}\right)
^{\ast}Y_{1},\ldots,\left(  T_{a_{R},N_{R}}^{-1}\right)  ^{\ast}Y_{R}\right)
\in%
{\displaystyle\bigoplus\limits_{i=1}^{R}}
\mathcal{D}^{L_{i}+\Delta_{i}}\left(  a_{i}+\mathfrak{\wp}^{N_{i}}%
\mathcal{O}_{\mathbb{K}}\right)  $, such that $\left\Vert \boldsymbol{f}%
-\boldsymbol{Y}\right\Vert _{\rho}<\epsilon$ for any $\boldsymbol{w}$,
$\boldsymbol{\theta}$, running in some balls.
\end{proof}

\section{Approximation of functions from $L^{\rho}\left(  \left[  0,1\right]
\right)  $ by non-Archimedean DNNs}

The following result is well-known in the mathematical folklore, but we cannot
find a suitable bibliographic reference for it.

\begin{lemma}
\label{Leemma_technical}Given $x\in\left[  0,1\right]  $, there exists a
unique sequence $\left\{  x_{i}\right\}  _{i\in\mathbb{N}}$, $x_{i}\in\left\{
0,\ldots,p-1\right\}  $, of $p$-adic digits such that%
\[
x=\sum_{i=0}^{\infty}x_{i}p^{-i-1}.
\]
The sequence of $p$-adic digits is computed recursively. If $x=\frac{m}{p^{n}%
}$, with $n$, $m=\sum_{i=0}^{k}m_{i}p^{i}\in\mathbb{N}\smallsetminus\left\{
0\right\}  $, $x_{i}\in\left\{  0,\ldots,p-1\right\}  $, and $0<m<p^{n}$, then%
\[
x=\sum_{i=0}^{k}m_{i}p^{i-n},
\]
which means that almost all the digits are zero.
\end{lemma}

\begin{proof}
The proof is a recursive algorithm for computing the $p$-adic digits of $x$.

\textbf{Step 0.} The digit $x_{0}=i$ if
\begin{equation}
\frac{i}{p}\leq x<\frac{i+1}{p}\text{, for }i\in\left\{  0,\ldots,p-1\right\}
. \label{Problem 1}%
\end{equation}
\textbf{Induction step}. Assume that the digits $x_{0}$,\ldots,$x_{N}$, for
some $N\geq1$, are already computed. Set $\widetilde{x}_{N}:=\sum_{i=0}%
^{N}x_{i}p^{-i-1}$,
\[
0\leq x-\widetilde{x}_{N}\leq\frac{1}{p^{N+1}}.
\]
If $x=\widetilde{x}_{N}$, then $x_{i}=0$ for $i\geq N+1$. If $x=\widetilde
{x}_{N}+\frac{1}{p^{N+1}}$, then $x_{i}=0$ for $i\geq N+2$.\ If
$0<x-\widetilde{x}_{N}<\frac{1}{p^{N+1}}$, we divide the interval $\left(
0,\frac{1}{p^{N+1}}\right)  $ into $p$ subintervals of length $\frac{1}{p}$,
and set $x_{N+1}=i$ if
\begin{equation}
\frac{i}{p^{N+2}}\leq x-\widetilde{x}_{N}<\frac{i+1}{p^{N+2}}\text{, for for
}i\in\left\{  0,\ldots,p-1\right\}  . \label{Problem 2}%
\end{equation}
We now replace $x$ by $\widetilde{x}_{N+1}$ and go back to Step 0. The
solutions of the decision problems (\ref{Problem 1})-(\ref{Problem 2}) are
uniquely determined by the number $x$. Consequently, the sequence of $p$-adic
digits is uniquely determined by $x$.
\end{proof}

We set%
\[
\mathcal{I}_{p}:=\left\{  \sum_{i=0}^{\infty}a_{i}p^{-i-1};a_{i}\in\left\{
0,\ldots,p-1\right\}  \right\}  .
\]
Given $x\in\left[  0,1\right]  $, we set $\varrho_{p}\left(  x\right)
:=\sum_{i=0}^{\infty}x_{i}p^{-i-1}$. By Lemma \ref{Leemma_technical},
$\varrho_{p}:\left[  0,1\right]  \rightarrow\mathcal{I}_{p}$ is a one-to-one
function. But, this function is not surjective. Indeed, the formula%
\[
\frac{1}{p^{N+1}}=\sum_{i=N+1}^{\infty}\left(  p-1\right)  p^{-i-1}%
\]
implies that
\[
\sum_{i=0}^{N-1}x_{i}p^{-i-1}+\frac{1}{p^{N+1}}=\sum_{i=0}^{N-1}x_{i}%
p^{-i-1}+\sum_{i=N+1}^{\infty}\left(  p-1\right)  p^{-i-1}.
\]
For $x=\sum_{i=0}^{N-1}x_{i}p^{-i-1}+\frac{1}{p^{N+1}}$, function $\varrho
_{p}$ gives the sequence $\left(  x_{0},\ldots,x_{N-1},1\right)  $; the
sequence $\left(  x_{0},\ldots,x_{N-1},p-1,p-1,\ldots\right)  $ is not in the
range of $\varrho_{p}$.

We now construct a\ generalization of the function $\varrho_{p}$ as follows:%
\[%
\begin{array}
[c]{cccc}%
\mathbb{\varrho}_{\mathbb{K}}: & \left[  0,1\right]  & \rightarrow &
\mathcal{O}_{\mathbb{K}}\\
&  &  & \\
& x=\sum_{i=0}^{\infty}x_{i}p^{-i-1} & \rightarrow & \mathbb{\varrho
}_{\mathbb{K}}\left(  x\right)  =\sum_{i=0}^{\infty}x_{i}\mathfrak{\wp}^{i}.
\end{array}
\]
This function is one-to-one, but not surjective. We denote the range of
$\mathbb{\varrho}_{\mathbb{K}}$ as $\mathcal{O}_{\mathbb{K}}\smallsetminus
\mathcal{M}$. The following result follows directly from the definition of the
function $\mathbb{\varrho}_{\mathbb{K}}$.

\begin{lemma}
\label{Lemma2}With the above notation, the following assertions hold:

\noindent(i) $\mathbb{\varrho}_{\mathbb{K}}\left(  x\right)  $ is a continuous function.

\noindent(ii) $\left\vert \mathbb{\varrho}_{\mathbb{K}}^{-1}\left(  x\right)
-\mathbb{\varrho}_{\mathbb{K}}^{-1}\left(  y\right)  \right\vert
\leq\left\vert x-y\right\vert _{\mathbb{K}}$, for $x,y\in\mathcal{O}%
_{\mathbb{K}}\smallsetminus\mathcal{M}$.

\noindent(iii) For any $b=\sum_{i=0}^{k}x_{i}\mathfrak{\wp}^{i}\in
\mathcal{O}_{\mathbb{K}}$,
\[
\mathbb{\varrho}_{\mathbb{K}}\left(  \sum_{i=0}^{k}x_{i}p^{-i-1}%
+p^{-k-1}\left[  0,1\right]  \right)  =\left(  b+\mathfrak{\wp}^{k+1}%
\mathcal{O}_{\mathbb{K}}\right)  \smallsetminus\mathcal{M}\text{, for }%
k\geq0\text{.}%
\]

\noindent(iv) $\mathbb{\varrho}_{\mathbb{K}}\left(  \left[  0,1\right]
\right)  =\mathcal{O}_{\mathbb{K}}\smallsetminus\mathcal{M}$.
\end{lemma}

We denote by $\mathcal{B}\left(  \left[  0,1\right]  \right)  $, respectively
$\mathcal{B}\left(  \mathcal{O}_{\mathbb{K}}\right)  $, the Borel $\sigma
$-algebra of $\left[  0,1\right]  $, respectively of $\mathcal{O}_{\mathbb{K}%
}$. By Lemma \ref{Lemma2}-(i), the function
\[
\mathbb{\varrho}_{\mathbb{K}}:\left(  \left[  0,1\right]  ,\mathcal{B}\left(
\left[  0,1\right]  \right)  \right)  \rightarrow\left(  \mathcal{O}%
_{\mathbb{K}},\mathcal{B}\left(  \mathcal{O}_{\mathbb{K}}\right)  \right)
\]
is measurable. We denote by $\mathfrak{L}:\mathcal{B}\left(  \left[
0,1\right]  \right)  \rightarrow\left[  0,+\infty\right]  $ the Lebesgue
measure in $\left[  0,1\right]  $, and by $\mathfrak{H}$ the pushforward
measure of $\mathfrak{L}$ to $\mathcal{B}\left(  \mathcal{O}_{\mathbb{K}%
}\right)  $:%
\[
\mathfrak{H}(B):=\mathfrak{L}\left(  \mathbb{\varrho}_{\mathbb{K}}^{-1}\left(
B\right)  \right)  \text{ for }B\in\mathcal{B}\left(  \mathcal{O}_{\mathbb{K}%
}\right)  .
\]
We denote by $\mathcal{B}_{\mathcal{M}}\left(  \mathcal{O}_{\mathbb{K}%
}\right)  $ the restriction of $\mathcal{B}\left(  \mathcal{O}_{\mathbb{K}%
}\right)  $ to $\mathcal{O}_{\mathbb{K}}\smallsetminus\mathcal{M}$, which is
the $\sigma$-algebra of subsets of the form%
\[
B\cap\left(  \mathcal{O}_{\mathbb{K}}\smallsetminus\mathcal{M}\right)  \text{,
for }B\in\mathcal{B}\left(  \mathcal{O}_{\mathbb{K}}\right)  .
\]
The function $\mathbb{\varrho}_{\mathbb{K}}^{-1}:\left(  \left[  0,1\right]
,\mathcal{B}\left(  \left[  0,1\right]  \right)  \right)  \rightarrow\left(
\mathcal{O}_{\mathbb{K}},\mathcal{B}_{\mathcal{M}}\left(  \mathcal{O}%
_{\mathbb{K}}\right)  \right)  $ is also measurable, cf. Lemma \ref{Lemma2}-(ii).

\begin{proposition}
\label{Prop1}The measure $\mathfrak{H}$ equals the Haar measure $\mu
_{\text{Haar}}$ in $\mathcal{B}\left(  \mathcal{O}_{\mathbb{K}}\right)  $.
Furthermore, $\mu_{\text{Haar}}\left(  \mathcal{M}\right)  =0$.
\end{proposition}

\begin{proof}
By Lemma \ref{Lemma2}-(iii)-(iv),%
\[
\mathfrak{H}\left(  \left(  b+\mathfrak{\wp}^{n}\mathcal{O}_{\mathbb{K}%
}\right)  \smallsetminus\mathcal{M}\right)  =\mu_{\text{Haar}}\left(
b+\mathfrak{\wp}^{n}\mathcal{O}_{\mathbb{K}}\right)  =p^{-n}\text{ for }%
n\in\mathbb{N}\text{, }b\in\mathcal{O}_{\mathbb{K}}\text{.}%
\]
By the Carath\'{e}odory extension theorem, see, e.g., \cite[Section
1.3.10]{Ash}, the Haar measure is uniquely determined by the condition
$\mu_{\text{Haar}}\left(  b+\mathfrak{\wp}^{n}\mathcal{O}_{\mathbb{K}}\right)
=p^{-n}$ for $n\in\mathbb{N}$, $b\in\mathcal{O}_{\mathbb{K}}$. Therefore,
$\mathfrak{H}$ extends to the Haar measure in $\mathcal{B}\left(
\mathcal{O}_{\mathbb{K}}\right)  $, and $\mu_{\text{Haar}}\left(
\mathcal{M}\right)  =0$.
\end{proof}

Given a function $f$ $:\mathcal{O}_{\mathbb{K}}\rightarrow\mathbb{R}$, we
define $\mathbb{\varrho}_{\mathbb{K}}^{\ast}f:\left[  0,1\right]
\rightarrow\mathbb{R}$, by $\left(  \mathbb{\varrho}_{\mathbb{K}}^{\ast
}f\right)  \left(  x\right)  =f\left(  \mathbb{\varrho}_{\mathbb{K}}\left(
x\right)  \right)  $.

\begin{theorem}
\label{Theorem2}The map $\mathbb{\varrho}_{\mathbb{K}}^{\ast}:L^{\rho}\left(
\mathcal{O}_{\mathbb{K}}\right)  \rightarrow L^{\rho}\left(  \left[
0,1\right]  \right)  $, $\rho\in\left[  1,\infty\right]  $, is an isometric
isomorphism of normed spaces, i.e., a linear surjective isometry.
\end{theorem}

\begin{proof}
Changing variables as $y=\varrho_{\mathbb{K}}\left(  x\right)  $, the Lebesgue
measure of $\left[  0,1\right]  $ becomes the Haar measure of $\mathcal{O}%
_{\mathbb{K}}$, then
\[
\left\Vert \mathbb{\varrho}_{\mathbb{K}}^{\ast}f\right\Vert _{\rho}=\left(
\text{ }%
{\displaystyle\int\limits_{\left[  0,1\right]  }}
\left\vert f\left(  \varrho_{\mathbb{K}}\left(  x\right)  \right)  \right\vert
^{\rho}dx\right)  ^{\frac{1}{\rho}}=\left(  \text{ }%
{\displaystyle\int\limits_{\mathcal{O}_{\mathbb{K}}}}
\left\vert f\left(  y\right)  \right\vert ^{\rho}dy\right)  ^{\frac{1}{\rho}%
}=\left\Vert f\right\Vert _{\rho}.
\]

\end{proof}

Theorems \ref{Theorem1}, \ref{Theorem2} imply that non-Archimedean DNNs can
approximate functions from \ $L^{\rho}\left(  \left[  0,1\right]  \right)  $
But, the spaces $\mathcal{D}^{L}\left(  \mathcal{O}_{\mathbb{K}}\right)  $ are
naturally isomorphic to spaces of simple functions defined on $\left[
0,1\right]  $; this isomorphism allows us to describe the approximation of
functions from $L^{\rho}\left(  \left[  0,1\right]  \right)  $ $\ $by DNNs in
terms of functions defined on $\left[  0,1\right]  $. The details are as follows.

\begin{remark}
Let $\left(  \Omega,\mathcal{F},\mu\right)  $ be a measure space, where
$\mathcal{F}$ is a $\sigma$-algebra of subsets of $\Omega$, and $\mu$ is
measure in $\mathcal{F}$. A function $h:\Omega\rightarrow\mathbb{R}$ is called
simple if $h(x)=\sum_{i=1}^{n}\alpha_{i}1_{A_{i}}\left(  x\right)  $, where
the $\alpha_{i}$ are real numbers, and the $A_{i}$ are disjoint sets in
$\mathcal{F}$. The simple functions are dense in $L^{\rho}\left(
\Omega,\mathcal{F},\mu\right)  $, for $\rho\in\left[  1,\infty\right]  $; see
\cite[Theorem 2.4.13]{Ash}.
\end{remark}

\begin{lemma}
\label{Lemma3}$\mathbb{\varrho}_{\mathbb{K}}$ (considered as a change of
variables) gives a bijection between the simple functions defined on
$\mathcal{O}_{\mathbb{K}}$ into the simple functions defined on $\left[
0,1\right]  $.
\end{lemma}

\begin{proof}
Take $B\in\mathcal{B}(\mathcal{O}_{\mathbb{K}})$, since $\mu_{\text{Haar}%
}\left(  \mathcal{M}\right)  =0$,%
\[
1_{B\smallsetminus\mathcal{M}}=1_{B}\text{ in }L^{\rho}\left(  \mathcal{O}%
_{\mathbb{K}}\right)  \text{.}%
\]
By using that $\mathbb{\varrho}_{\mathbb{K}}$ is measurable, cf. Lemma
\ref{Lemma2}-(i),
\[
\mathbb{\varrho}_{\mathbb{K}}^{-1}\left(  B\smallsetminus\mathcal{M}\right)
=\mathbb{\varrho}_{\mathbb{K}}^{-1}\left(  B\right)  \smallsetminus
\mathbb{\varrho}_{\mathbb{K}}^{-1}\left(  \mathcal{M}\right)
\]
is a Borel subset of $\left[  0,1\right]  $. Now the Lebesgue measure of
$\mathbb{\varrho}_{\mathbb{K}}^{-1}\left(  \mathcal{M}\right)  $ is zero, cf.
Proposition \ref{Prop1}, and thus,%
\[
1_{\mathbb{\varrho}_{\mathbb{K}}^{-1}\left(  B\right)  \smallsetminus
\mathbb{\varrho}_{\mathbb{K}}^{-1}\left(  \mathcal{M}\right)  }%
=1_{\mathbb{\varrho}_{\mathbb{K}}^{-1}\left(  B\right)  }\text{ in }L^{\rho
}\left(  \left[  0,1\right]  \right)  .
\]
Consequently, $\mathbb{\varrho}_{\mathbb{K}}^{-1}$ sends the function $1_{B}$
into $1_{\mathbb{\varrho}_{\mathbb{K}}^{-1}\left(  B\right)  }$. We now take a
Borel subset $A$ of $\left[  0,1\right]  $. By using that $\mathbb{\varrho
}_{\mathbb{K}}^{-1}$ is measurable, cf. Lemma \ref{Lemma2}-(ii),
$\mathbb{\varrho}_{\mathbb{K}}\left(  A\right)  $ is a Borel subset of
$\mathcal{O}_{\mathbb{K}}$. Then, $\mathbb{\varrho}_{\mathbb{K}}$ sends the
function $1_{A} $ \ into $1\mathbb{\varrho}_{\mathbb{K}}\left(  A\right)  $.
\end{proof}

Set%
\[
\mathbb{\varrho}_{\mathbb{K}}^{-1}\left(  G_{l}\right)  =\left\{  a_{0}%
p^{-1}+a_{1}p^{-2}+\ldots+a_{l-1}p^{-l-2};a_{i}\in\left\{  0,\ldots
,p-1\right\}  \right\}  .
\]
Notice that $\mathbb{\varrho}_{\mathbb{K}}^{-1}\left(  G_{l}\right)  $ is not
an additive group. For instance,
\[
1p^{-l-2},\left(  p-1\right)  p^{-l-2}\in\mathbb{\varrho}_{\mathbb{K}}%
^{-1}\left(  G_{l}\right)  \text{, but }1p^{-l-2}+\left(  p-1\right)
p^{-l-2}=p^{-l-1}\notin\mathbb{\varrho}_{\mathbb{K}}^{-1}\left(  G_{l}\right)
.
\]

Now, $\phi\in\mathcal{D}^{L}\left(  \mathcal{O}_{\mathbb{K}}\right)  $ can be
identified with the function%
\[
\phi\left(  x\right)  =%
{\displaystyle\sum\limits_{\alpha\in\mathbb{\varrho}_{\mathbb{K}}^{-1}\left(
G_{l}\right)  }}
\phi\left(  \alpha\right)  1_{\alpha+\left[  0,p^{-l-1}\right]  }\left(
x\right)  ,\text{ for }x\in\left[  0,1\right]  .
\]
Now, by Theorems \ref{Theorem1} and \ref{Theorem2}, and Lemma \ref{Lemma3}, we
have the following result.

\begin{theorem}
\label{Theorem3A}Given any $f\in L^{\rho}\left(  \left[  0,1\right]  \right)
$, $1\leq\rho\leq\infty$, with $\left\Vert f\right\Vert _{\rho}<M$, and any
$\epsilon>0$; then, for any input $X\in\mathcal{D}^{L}\left(  \mathcal{O}%
_{\mathbb{K}}\right)  $, $L\geq1$, there exists a
\[
DNN(\mathcal{O}_{\mathbb{K}},\sigma_{M},L,\Delta,w,\theta)
\]
with output
\[
Y(x;\mathcal{O}_{\mathbb{K}},\sigma_{M},L,\Delta,w,\theta)=%
{\displaystyle\sum\limits_{\alpha\in\mathbb{\varrho}_{\mathbb{K}}^{-1}\left(
G_{L+\Delta}\right)  }}
Y\left(  \alpha\right)  1_{\alpha+\left[  0,p^{-L-\Delta-1}\right]  }\left(
x\right)
\]
satisfying
\[
\left\Vert Y\left(  \mathcal{O}_{\mathbb{K}},\sigma_{M},L,\Delta
,w,\theta\right)  -f\right\Vert _{\rho}<\epsilon,
\]
for any $w$ in a ball in $\mathbb{R}^{2p^{\left(  L+\Delta\right)  }}$, and
$\ $any $\theta$ in a ball in $\mathbb{R}^{p^{\left(  L+\Delta\right)  }}$.
\end{theorem}

\section{Fourier-Walsh series}

In \cite{Abbet et al}, the authors use the Fourier-Walsh series to approximate
hierarchical functions. Here, we show that this is a pure non-Archimedean technique.

\subsection{Additive characters}

An additive character of $\mathbb{K}$ is a continuous group homomorphism from
$\left(  \mathbb{K},+\right)  $ into $\left(  S,\cdot\right)  $, the complex
unit circle considered as a multiplicative group. We denote by $\chi
_{\text{triv}}$ by the trivial additive character of $\mathbb{K}$, i.e.,
$\chi_{\text{triv}}\left(  x\right)  =1$ for any $x\in\mathbb{K}$. We say
$\chi_{_{\mathbb{K}}}$ is a standard additive character of $\mathbb{K}$ , if
$\chi_{_{\mathbb{K}}}\neq\chi_{\text{triv}}$, and $\chi_{_{\mathbb{K}}}\left(
x\right)  =1$ for any $x\in\mathcal{O}_{\mathbb{K}}$. In the case,
$\mathbb{K}=\mathbb{F}_{p}\left(  \left(  T\right)  \right)  $, for
$x=\sum_{k=l}^{\infty}a_{k}T^{k}$, $l\in\mathbb{Z}$, we set $Res(x):=a_{-1}$,
then, the function%
\[
\chi_{_{\mathbb{K}}}\left(  x\right)  =\exp\left(  \frac{2\pi i}%
{p}Res(x)\right)
\]
is a standard additive character of $\mathbb{F}_{p}\left(  \left(  T\right)
\right)  $. In the case $\mathbb{K}=\mathbb{Q}_{p}$, for $x=\sum_{k=r}%
^{\infty}a_{k}p^{k}$, $r\in\mathbb{Z}$, we set
\[
\left\{  x\right\}  _{p}:=\left\{
\begin{array}
[c]{lll}%
0 & \text{if} & r\geq0\\
&  & \\
\sum_{k=r}^{-1}a_{k}p^{k} & \text{if} & r<0.
\end{array}
\right.
\]
Then,%
\[
\chi_{_{\mathbb{K}}}\left(  x\right)  =\exp\left(  2\pi i\left\{  x\right\}
_{p}\right)
\]
is a standard additive character of $\mathbb{Q}_{p}$.

Any additive character of $\mathbb{K}$ has the form $\chi\left(  x\right)
=\chi_{_{\mathbb{K}}}\left(  ax\right)  $ for $x\in\mathbb{K}$, and some
$a\in\mathbb{K}$, see \cite[Chap. II, Corollary to Theorem 3]{We}. We denote
by $\Omega\left(  \mathcal{O}_{\mathbb{K}}\right)  $ the group of additive
characters of $\mathcal{O}_{\mathbb{K}}$. Then,
\[
\Omega\left(  \mathcal{O}_{\mathbb{K}}\right)  =\left\{  \chi_{_{\mathbb{K}}%
}\left(  \wp^{-l}\boldsymbol{a}x\right)  ;l\in\mathbb{N}\smallsetminus\left\{
0\right\}  ,\text{ }\boldsymbol{a}\in G_{l}\left(  \mathcal{O}_{\mathbb{K}%
}\right)  \right\}  \cup\left\{  \chi_{\text{triv}}\right\}  ,
\]
where $G_{l}\left(  \mathcal{O}_{\mathbb{K}}\right)  =\mathcal{O}_{\mathbb{K}%
}/\wp^{l}\mathcal{O}_{\mathbb{K}}$.

\subsection{Orthonormal basis of $L\left(  \mathcal{O}_{\mathbb{K}}\right)
{\textstyle\bigotimes}
\mathbb{C}$ and $L^{2}\left(  \left[  0,1\right]  \right)
{\textstyle\bigotimes}
\mathbb{C}$}

Given $f,g:\mathcal{O}_{\mathbb{K}}\rightarrow\mathbb{C}$, we set
\[
\left\langle f,g\right\rangle =%
{\displaystyle\int\limits_{\mathcal{O}_{\mathbb{K}}}}
f\left(  x\right)  \overline{g}\left(  x\right)  dx\text{, and }\left\Vert
f\right\Vert _{2}^{2}=\left\langle f,f\right\rangle ,
\]
where the bar denotes the complex conjugate. We set%
\[
L^{2}\left(  \mathcal{O}_{\mathbb{K}}\right)
{\textstyle\bigotimes}
\mathbb{C}=\left\{  f:\mathcal{O}_{\mathbb{K}}\rightarrow\mathbb{C};\left\Vert
f\right\Vert _{2}<\infty\right\}  .
\]
This space is the complexification of the space $L^{2}\left(  \mathcal{O}%
_{\mathbb{K}}\right)  $. We set $L^{2}\left(  \left[  0,1\right]  \right)
{\textstyle\bigotimes}
\mathbb{C}$ for the complexification of $L^{2}\left(  \left[  0,1\right]
\right)  $. We denote by $\mathcal{D}\left(  \mathcal{O}_{\mathbb{K}}\right)
{\textstyle\bigotimes}
\mathbb{C}$, the $\mathbb{C}$-vector space of test functions, which is the
complexification of $\mathcal{D}\left(  \mathcal{O}_{\mathbb{K}}\right)  $.

\begin{theorem}
\label{Theorem4}With the above notation, the following assertions hold.

\noindent(i) The group of characters $\Omega\left(  \mathcal{O}_{\mathbb{K}%
}\right)  $ is an orthonormal basis of $L^{2}\left(  \mathcal{O}_{\mathbb{K}%
}\right)
{\textstyle\bigotimes}
\mathbb{C}$.

\noindent(ii) Set%
\[
\omega_{\wp,\boldsymbol{a},l}\left(  x\right)  :=\chi_{_{\mathbb{K}}}\left(
\wp^{-l}\boldsymbol{a}\mathbb{\varrho}_{\mathbb{K}}\left(  x\right)  \right)
,\text{ for }x\in\left[  0,1\right]  \text{, }l\in\mathbb{N}\smallsetminus
\left\{  0\right\}  ,\boldsymbol{a}\in G_{l}\left(  \mathcal{O}_{\mathbb{K}%
}\right)  .
\]
Then
\[
\left\{  \omega_{\wp,\boldsymbol{a},l}\right\}  _{\wp,\boldsymbol{a},l}%
\cup\left\{  1_{\left[  0,1\right]  }\right\}
\]
is an orthonormal basis of $L^{2}\left(  \left[  0,1\right]  \right)
{\textstyle\bigotimes}
\mathbb{C}$.
\end{theorem}

\begin{proof}
We start by recalling the following fact: the group of characters
$\Omega\left(  \mathcal{O}_{\mathbb{K}}\right)  $ is an orthonormal basis of
the pre-Hilbert space $\left(  \mathcal{D}\left(  \mathcal{O}_{\mathbb{K}%
}\right)
{\textstyle\bigotimes}
\mathbb{C},\left\langle \cdot,\cdot\right\rangle \right)  $, i.e.,
\[
\left\langle \chi,\chi^{\prime}\right\rangle =\left\{
\begin{array}
[c]{ccc}%
0 & \text{if} & \chi\neq\chi^{\prime}\\
&  & \\
1 & \text{if} & \chi=\chi^{\prime},
\end{array}
\right.
\]
for any $\chi,\chi^{\prime}\in\Omega\left(  \mathcal{O}_{\mathbb{K}}\right)
$, and any function $\varphi\in\mathcal{D}\left(  \mathcal{O}_{\mathbb{K}%
}\right)
{\textstyle\bigotimes}
\mathbb{C}$ can be expressed as a finite sum%
\begin{equation}
\varphi\left(  x\right)  =%
{\displaystyle\sum\limits_{\chi\in\Omega\left(  \mathcal{O}_{\mathbb{K}%
}\right)  }}
C_{\chi}\chi\left(  x\right)  \text{, \ where }C_{\chi}=\left\langle
\varphi,\chi\right\rangle . \label{Form}%
\end{equation}
This\ assertion is a particular case of a well-known result of harmonic
analysis on compact groups, see \cite[Proposition 7.2.2]{Igusa}.

We denote by $Span\left(  \Omega\left(  \mathcal{O}_{\mathbb{K}}\right)
\right)  $, the $\mathbb{C}$-vector subspace of $L^{2}\left(  \mathcal{O}%
_{\mathbb{K}}\right)
{\textstyle\bigotimes}
\mathbb{C}$ spanned by the elements of $\Omega\left(  \mathcal{O}_{\mathbb{K}%
}\right)  $. To show that $\Omega\left(  \mathcal{O}_{\mathbb{K}}\right)  $ is
an orthonormal basis of $L^{2}\left(  \mathcal{O}_{\mathbb{K}}\right)
{\textstyle\bigotimes}
\mathbb{C}$, it is sufficient to establish that
\[
\overline{Span\left(  \Omega\left(  \mathcal{O}_{\mathbb{K}}\right)  \right)
}=L^{2}\left(  \mathcal{O}_{\mathbb{K}}\right)
{\textstyle\bigotimes}
\mathbb{C}\text{,}%
\]
where the bar denotes the topological closure of $Span\left(  \Omega\left(
\mathcal{O}_{\mathbb{K}}\right)  \right)  $ in $L^{2}\left(  \mathcal{O}%
_{\mathbb{K}}\right)
{\textstyle\bigotimes}
\mathbb{C}$. Let $f\in L^{2}\left(  \mathcal{O}_{\mathbb{K}}\right)
{\textstyle\bigotimes}
\mathbb{C}$ and $\epsilon>0$ given. Then, by the density of $\mathcal{D}%
\left(  \mathcal{O}_{\mathbb{K}}\right)
{\textstyle\bigotimes}
\mathbb{C}$ in $L^{2}\left(  \mathcal{O}_{\mathbb{K}}\right)
{\textstyle\bigotimes}
\mathbb{C}$ and (\ref{Form}), there exists $\varphi\in\mathcal{D}\left(
\mathcal{O}_{\mathbb{K}}\right)
{\textstyle\bigotimes}
\mathbb{C}$\ $\cap$ $Span\left(  \Omega\left(  \mathcal{O}_{\mathbb{K}%
}\right)  \right)  $ such that $\left\Vert f-\varphi\right\Vert _{2}<\epsilon$.

(ii) By Theorem \ref{Theorem2}, the map $\mathbb{\varrho}_{\mathbb{K}}^{\ast
}:L^{2}\left(  \mathcal{O}_{\mathbb{K}}\right)  \rightarrow L^{2}\left(
\left[  0,1\right]  \right)  $, is a linear surjective isometry, i.e.,
$\left\Vert g\right\Vert _{2}^{2}=\left\Vert \mathbb{\varrho}_{\mathbb{K}%
}^{\ast}g\right\Vert _{2}^{2}$ for any $g\in L^{2}\left(  \mathcal{O}%
_{\mathbb{K}}\right)  $. Take $f=f_{1}+if_{2}\in L^{2}\left(  \mathcal{O}%
_{\mathbb{K}}\right)
{\textstyle\bigotimes}
\mathbb{C}$, with $f_{i}\in L^{2}\left(  \mathcal{O}_{\mathbb{K}}\right)  $,
$i=1,2$. We now set $\mathbb{\varrho}_{\mathbb{K}}^{\ast}f=\mathbb{\varrho
}_{\mathbb{K}}^{\ast}f_{1}+i\mathbb{\varrho}_{\mathbb{K}}^{\ast}f_{2}\in
L^{2}\left(  \left[  0,1\right]  \right)
{\textstyle\bigotimes}
\mathbb{C}$. Then
\[
\left\Vert \mathbb{\varrho}_{\mathbb{K}}^{\ast}f\right\Vert _{2}%
^{2}=\left\Vert \mathbb{\varrho}_{\mathbb{K}}^{\ast}f_{1}\right\Vert _{2}%
^{2}+\left\Vert \mathbb{\varrho}_{\mathbb{K}}^{\ast}f_{2}\right\Vert _{2}%
^{2}=\left\Vert f_{1}\right\Vert _{2}^{2}+\left\Vert f_{2}\right\Vert _{2}%
^{2}=\left\Vert f\right\Vert _{2}^{2}.
\]
Consequently, $\mathbb{\varrho}_{\mathbb{K}}^{\ast}:L^{2}\left(
\mathcal{O}_{\mathbb{K}}\right)
{\textstyle\bigotimes}
\mathbb{C}\rightarrow L^{2}\left(  \left[  0,1\right]  \right)
{\textstyle\bigotimes}
\mathbb{C}$ is an $L^{2}$-isometry, and by the polarization identities,
\[
\left\langle \mathbb{\varrho}_{\mathbb{K}}^{\ast}f,\mathbb{\varrho
}_{\mathbb{K}}^{\ast}g\right\rangle =\left\langle f,g\right\rangle \text{ for
any }f,g\in L^{2}\left(  \mathcal{O}_{\mathbb{K}}\right)
{\textstyle\bigotimes}
\mathbb{C}%
\]
Therefore, $\mathbb{\varrho}_{\mathbb{K}}^{\ast}$ sends orthonormal bases of
$L^{2}\left(  \mathcal{O}_{\mathbb{K}}\right)
{\textstyle\bigotimes}
\mathbb{C}$ into orthonormal bases of $L^{2}\left(  \left[  0,1\right]
\right)
{\textstyle\bigotimes}
\mathbb{C}$.
\end{proof}

Besides the rings $\mathbb{F}_{p}\left[  \left[  T\right]  \right]  $,
$\mathbb{Z}_{p}$ share many common \ arithmetic and geometric features, their
arithmetic operations are quite different. In the case of $\mathbb{F}%
_{p}\left[  \left[  T\right]  \right]  $, \ given $x=\sum_{i=0}^{\infty}%
x_{i}T^{i}$, $y=\sum_{i=0}^{\infty}y_{i}T^{i}\in\mathbb{F}_{p}\left[  \left[
T\right]  \right]  $,
\begin{equation}
x+y=\sum_{i=0}^{\infty}\left(  x_{i}+y_{i}\right)  T^{i},\text{ and }%
xy=\sum_{i=0}^{\infty}\left(
{\displaystyle\sum\limits_{s=0}^{i}}
x_{s}y_{i-s}\right)  T^{i}, \label{Formulas}%
\end{equation}
where the digit operations are performed in the field $\mathbb{F}_{p}$. In
particular, the addition in $\mathbb{F}_{p}\left[  \left[  T\right]  \right]
$ does not require carry digits. Using (\ref{Formulas}), we rewrite the
orthonormal basis given in Theorem \ref{Theorem4} as follows. The set of
functions%
\begin{equation}
\theta_{p,\boldsymbol{n},l}\left(  x\right)  :=\exp\left(  \frac{2\pi i}{p}%
{\displaystyle\sum\limits_{i=0}^{l}}
n_{l-1-i}x_{i}\right)  , \label{Theta_basis}%
\end{equation}
where \
\[
x=%
{\displaystyle\sum\limits_{i=0}^{\infty}}
x_{i}p^{-i-1}\in\left[  0,1\right]  \text{, }\boldsymbol{n}=%
{\displaystyle\sum\limits_{i=0}^{l}}
n_{i}p^{i-1}\in\mathbb{N}\text{, }n_{i}\in\left\{  0,\ldots,p-1\right\}
\text{, }l\geq1,
\]
form an orthonormal basis of $L^{2}\left(  \left[  0,1\right]  \right)
{\textstyle\bigotimes}
\mathbb{C}$. The digit operations in (\ref{Theta_basis}) are performed in the
field $\mathbb{F}_{p}$.

In the case of $\mathbb{Z}_{p}$, \ given $x=\sum_{i=0}^{\infty}x_{i}p^{i}$,
$y=\sum_{i=0}^{\infty}y_{i}p^{i}\in\mathbb{Z}_{p}$, there are no closed
formulas for the digits of the sum $x+y=\sum_{i=0}^{\infty}a_{i}p^{i}$ or the
product $xy=\sum_{i=0}^{\infty}b_{i}p^{i}$ because the addition in
$\mathbb{Z}_{p}$ requires carry digits. The digits of the sum and product are
compute recursively:%
\begin{equation}
\sum_{i=0}^{L-1}a_{i}p^{i}\equiv\sum_{i=0}^{L-1}x_{i}p^{i}+\sum_{i=0}%
^{L-1}y_{i}p^{i}\text{ }\operatorname{mod}p^{L}\text{, for any }L\geq1,
\label{Formulas1}%
\end{equation}
and
\begin{equation}
\sum_{i=0}^{L-1}b_{i}p^{i}\equiv\left(  \sum_{i=0}^{L-1}x_{i}p^{i}\right)
\left(  \sum_{i=0}^{L-1}y_{i}p^{i}\right)  \text{ }\operatorname{mod}%
p^{L}\text{, for any }L\geq1, \label{Formulas2}%
\end{equation}
where for $s,t\in\mathbb{Z}$, $s\equiv t\operatorname{mod}p^{L}$ means that
$s-t$ is divisible by $p^{L}$. We now rewrite the orthonormal basis given in
Theorem \ref{Theorem4} as follows. The set of functions%
\begin{equation}
\gamma_{p,\boldsymbol{n},l}\left(  x\right)  :=\exp\left(  \frac{2\pi i}%
{p^{l}}\left(
{\displaystyle\sum\limits_{i=0}^{l-1}}
n_{i}p^{i}\right)  \left(
{\displaystyle\sum\limits_{i=0}^{l-1}}
x_{i}p^{i}\right)  \right)  , \label{Gamma-Basis}%
\end{equation}
where the digit operations indicated in (\ref{Gamma-Basis}) are performed in
the quotient ring $\mathbb{Z}/p^{l}\mathbb{Z}$, i.e., they are arithmetic
operations $\operatorname{mod}p^{l}$, and%
\[
x=%
{\displaystyle\sum\limits_{i=0}^{\infty}}
x_{i}p^{-i-1}\in\left[  0,1\right]  \text{, }\boldsymbol{n}=%
{\displaystyle\sum\limits_{i=0}^{l}}
n_{i}p^{i-1}\in\mathbb{N}\text{, }n_{i}\in\left\{  0,\ldots,p-1\right\}
\text{, }l\geq1,
\]
form an orthonormal basis of $L^{2}\left(  \left[  0,1\right]  \right)
{\textstyle\bigotimes}
\mathbb{C}$.

The above described orthonormal bases of Walsh-Paley type, see, e.g.,
\cite{Chrestenson}, \cite{Fine}, \cite{Paley}. In the case $\mathbb{F}%
_{2}\left[  \left[  T\right]  \right]  $, the basis (\ref{Theta_basis}) agrees
with Walsh-Paley system \cite{Paley}; in this framework the connection with
the group of characters $\Omega\left(  \mathbb{F}_{2}\left[  \left[  T\right]
\right]  \right)  $ is well-known, \cite{Fine}. There are generalized
Walsh-Paley bases, see, e.g. \cite{Chrestenson}, however, the digit operations
used in \cite{Chrestenson}\ are different to the ones used here. On the other
hand. in the last thirty-five years orthonormal bases on $L^{2}\left(
\mathbb{Q}_{p}\right)
{\textstyle\bigotimes}
\mathbb{C}$ has been studied intensively; see, e.g., \cite{A-K-S},
\cite{KKZuniga}, \cite{Kochubei}, \cite{V-V-Z}, and the references therein.

\end{document}